\documentclass[sigconf]{acmart}

\usepackage{booktabs} 

\setcopyright{none}

\usepackage{wrapfig,lipsum,booktabs}
\usepackage{enumitem}
\usepackage{graphicx}
\usepackage[para,online,flushleft]{threeparttable}
\usepackage{caption}
\usepackage{pifont}
\usepackage{booktabs}
 \usepackage[T1]{fontenc}
\usepackage{enumitem}
\usepackage{lipsum}
\setlist[itemize]{leftmargin=*}
\usepackage{amssymb}
\usepackage{pifont}
 \usepackage[flushleft]{threeparttable}
\usepackage{adjustbox}
\usepackage{array}
\usepackage{amsthm}
\usepackage{multirow}
\newtheorem{remark}{Remark}

\newcolumntype{R}[2]{%
    >{\adjustbox{angle=#1,lap=\width-(#2)}\bgroup}%
    l%
    <{\egroup}%
}

\usepackage{algorithmic}
\usepackage{algorithm}

\DeclareMathOperator*{\argmax}{arg\,max}
\begin{document}
\title{A Generic Framework for Interesting Subspace Cluster Detection in Multi-attributed Networks}

 \author{Feng Chen, Baojian Zhou *, Adil Alim *}
 \affiliation{%
   \institution{University at Albany}
 }
 \email{{fchen5, aalimu, bzhou6}@albany.edu}

 \author{Liang Zhao}
 \affiliation{%
   \institution{George Manson University}
 }
 \email{zhaoliangvaio@gmail.com}

\thanks{* These two authors contributed equally}


\begin{abstract}

Detection of interesting (e.g., coherent or anomalous) clusters has been studied extensively on plain or univariate networks, with various applications. Recently, algorithms have been extended to networks with multiple attributes for each node in the real-world. In a multi-attributed network, often, a cluster of nodes is only interesting for \underline{\smash{a subset (subspace)}} of attributes, and this type of clusters is called \textit{subspace clusters}. However, in the current literature, few methods are capable of detecting  subspace clusters, 
which involves concurrent feature selection and network cluster detection. These relevant methods are mostly heuristic-driven  and customized for specific application scenarios. 

In this work, we present a generic and theoretical framework for detection of interesting subspace clusters in large multi-attributed networks. 
Specifically, we propose a subspace graph-structured matching pursuit algorithm, namely, \texttt{SG-Pursuit}, to address a broad class of such problems for different score functions (e.g., coherence or anomalous functions) and topology constraints (e.g., connected subgraphs and dense subgraphs). 
We prove that our algorithm 1) runs in nearly-linear time on the network size and the total number of attributes and 2) enjoys rigorous guarantees (geometrical convergence rate and tight error bound) analogous to those of the state-of-the-art algorithms for sparse feature selection problems and subgraph detection problems. As a case study, we specialize \texttt{SG-Pursuit} to optimize a number of well-known score functions for two typical tasks, including detection of coherent dense and anomalous connected subspace clusters in real world networks. Empirical evidence demonstrates that our proposed
generic algorithm \texttt{SG-Pursuit} performs superior over state-of-the-art methods that are designed specifically for these two tasks. 

\vspace{-3mm}
\end{abstract}

\maketitle

\section{Introduction}
With recent advances in hardware and software technologies, the huge volumes of data now being collected from multiple sources are naturally modeled as \textit{multi-attributed networks}. For example, massive multi-attributed biological networks have been created by integrating gene expression data with secondary data such as pathway or protein-protein interaction data for improved outcome prediction of cancer patients~\cite{mihaly2013meta}. Other examples include the multi-attributed networks that combine ``Big data'' (e.g., Twitter feeds) and traditional surveillance data for influenza studies~\cite{davidson2015using} and the social networks that contain both the friendship relations  and user attributes such as interests, frequencies of keywords mentioned in posts, and demographics~\cite{gunnemann2014gamer}. 

\begin{figure}[t]
\center
\includegraphics[width=8.5cm]{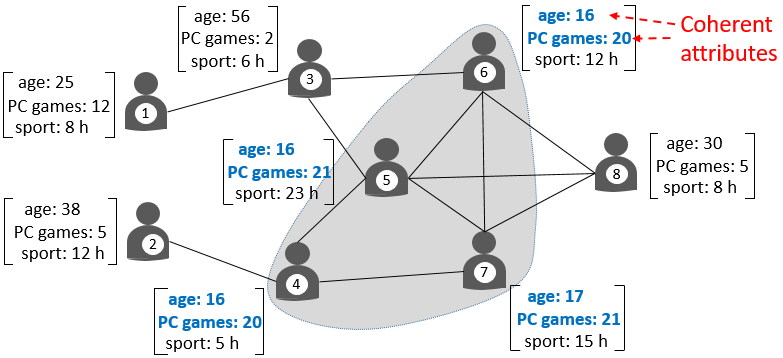}
\captionof{figure}{A social network with three attributes (age, PC games, and sport) in each user node and one potential coherent dense subspace cluster (highlighted in the shaded region and blue-colored texts) that has a coherent subset of attributes (age and PC games) and a dense subgraph of nodes (4, 5, 6, and 7). This cluster might be of interest to video game producers. (Adapt from \cite{gunnemann2014gamer}) }\label{fig:dense-coherent-subgraphs}
\vspace{-13pt}
  \end{figure}
  
  \begin{figure}[t]
\center
\includegraphics[width=8.5cm]{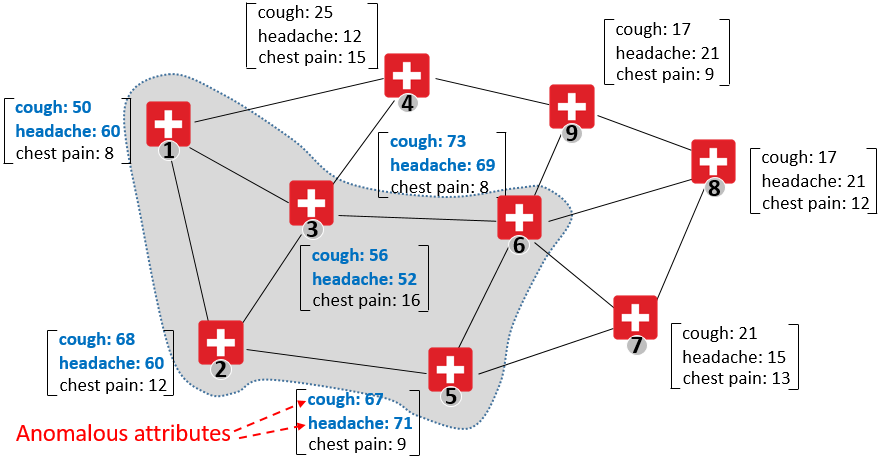}
\captionof{figure}{A health surveillance network of emergency departments (EDs) with three attributes (counts of cases of three different ICD-9 disease symptoms~\cite{neill2013fast}, including cough, headache, chest pain) in each ED node and one potential anomalous connected subspace cluster (highlighted in the shaded region and blue-colored texts) that has a anomalous subset of attributes (cough and headache) and a connected subgraph of nodes (1, 2, 3, 5, and 6).  The counts of these two attributes within the subgraph are abnormally higher than those outside the subgraph. In this scenario, the anomalous connected subspace cluster is used for disease outbreak detection.} \label{fig:dense-coherent-subgraphs}
\vspace{-10pt}  
\end{figure}
  
  \begin{table*}[t]
\small
\caption{Comparison of related work (``Generality'' refers to the capability of a method to support different score functions and topological constraints on \underline{\smash{subspace clusters}} on attributed networks. ``Good tradeoff'' refers to the good trade-off between tractability and quality guarantee on \underline{\smash{subspace clusters}}, when the number of feasible subgraphs (neighborhoods) is large. 
}
\vspace{-3mm}
\centering
\newcommand*\rot{\multicolumn{1}{R{45}{1em}}}
\renewcommand*\rot[2]{\multicolumn{1}{R{#1}{#2}}}

\resizebox{\textwidth}{!}{
\begin{tabular}{m{7.5cm}|m{1.1cm} m{1.1cm} m{1.1cm} m{1.1cm} m{1.1cm} m{1.1cm} m{1.1cm}}\hline
&
\rot{25}{1em}{Cluster detection} &
\rot{25}{1em}{On attributed networks} &
\rot{25}{1em}{Attribute subspace} &
\rot{25}{1em}{Coherent dense clusters} &
\rot{25}{1em}{Anomalous connected clusters} &
\rot{25}{1em}{Generality} & 
\rot{25}{1em}{Good tradeoff}
    \\ \hline
METIS~\cite{karypis1998multilevel}, Spectral~\cite{ng2001spectral}, Co-clustering~\cite{dhillon2003information}, PageRank-Nibble~\cite{andersen2006local}        &  \checkmark      &       &    & &  &  &        \\  
PICS~\cite{akoglu2012pics}, CODA~\cite{gao2010community}        &  \checkmark      &   \checkmark    &    & &  &  &     \\  
NPHGS~\cite{chen2014non}, EDAN~\cite{rozenshtein2014event}, CSGN~\cite{qian2014connected}, GSSO~\cite{zhou2016icdm}, GSPA~\cite{chen2016generalized}        &  \checkmark      &   \checkmark    &    & &  \checkmark &  &     \\  
CoPaM~\cite{moser2009mining}, Gamer~\cite{gunnemann2010subspace, gunnemann2014gamer}, \texttt{FocusCO}~\cite{perozzi2014focused}, AW-NCut~\cite{gunnemann2013spectral}       &  \checkmark      &  \checkmark      & \checkmark    & \checkmark &  & &    \\  
SODA~\cite{gupta2014local}, AMEN~\cite{perozzi2016scalable}   &  \checkmark      &  \checkmark      & \checkmark    &  & \checkmark &  &     \\  \hline
\texttt{SG-Pursuit}  [this paper]     &  \checkmark      &  \checkmark      & \checkmark    &  \checkmark & \checkmark &  \checkmark &    \checkmark  \\  \hline
\end{tabular}
}
\label{table:comparison}
\vspace{-3 mm}
\end{table*}

As one of the major tasks
in network mining, the detection of interesting clusters in attributed networks, such as coherent or anomalous clusters, has attracted a great deal of attention
in many applications, including medicine and public health~\cite{gunnemann2014gamer, neill2013fast}, law enforcement~\cite{wu2010identifying}, cyber security~\cite{pasqualetti2013attack}, transportation~\cite{anbarouglu2015non}, among others~\cite{ramakrishnan2014beating, ranu2013mining, rozenshtein2014event}. To deal with the multiple or even high-dimensional attributes, most existing methods either utilize all the given attributes~\cite{gao2010community, noble2003graph} or preform a unsupervised feature selection as a preprocessing step~\cite{tang2012unsupervised}. However, as demonstrated in a number of studies~\cite{perozzi2014focused, perozzi2016scalable,  gunnemann2014gamer, gunnemann2013spectral, gunnemann2010subspace},  clusters of interest in a multi-attributed network are often \textbf{subspace clusters}, each of which is defined by \textbf{a cluster of nodes} and \textbf{a relevant subset of attributes}. For example, in social networks, it is very unlikely that people are similar within all of their characteristics~\cite{gunnemann2014gamer}. In health surveillance networks, it is very rare that outbreaks of different disease types have identical symptoms~\cite{neill2013fast}.  In order to detect subspace clusters, it is required to conduct feature selection and cluster detection, concurrently, as 
without knowing the true clusters of nodes, it is difficult to identify their relevant attributes, and vice versa.

In recent years, a limited number of methods have been proposed to detect \textit{subspace clusters}, which fall into two main categories, 
including detection of coherent dense subspace clusters and detection of anomalous connected subspace clusters. The methods for detecting \textbf{coherent dense subspace clusters} search for subsets of nodes that show high similarity in subsets of their attributes and that are as well densely connected within the input network. Customized algorithms are developed for specific combinations of similarity functions of attributes (e.g., threshold based~\cite{gunnemann2014gamer, gunnemann2013spectral} and pairwise distance based~\cite{perozzi2014focused} functions) and density functions of nodes~\cite{perozzi2014focused, gunnemann2014gamer, gunnemann2013spectral, gunnemann2010subspace, moser2009mining}. The methods for detecting \textbf{anomalous connected subspace clusters} search for subsets of nodes that are significantly different from the other nodes on subsets of their attributes and that are as well connected (but not necessary dense) within the input network. The connectivity constraint ensures that the clusters of nodes reflect changes due to localized in-network processes. All the existing methods in this category consider a small set of neighborhoods (e.g., social circles and ego networks~\cite{perozzi2016scalable}, subgraphs isomorphic to a query graph~\cite{gupta2014local}, and small-diameter subgraphs~\cite{neill2013fast}), and identify anomalous subspace clusters \textit{among only these given neighborhoods}. 

However, the aforementioned methods have two main limitations: 1) \textbf{Lack of generality.} All these methods are customized for specific score functions of attributes and topological constraints on clusters, and may be inapplicable if the functions or constraints are changed. 
As discussed in recent surveys~\cite{akoglu2015graph}, the definition of an interesting subgraph pattern, in which subspace clusters is a specific type, is meaningful only under a given context or application. There is a strong need of generic methods that can handle a broad class of score functions, such as parametric/nonparametric scan statistic functions~\cite{chen2014non}, discriminative functions~\cite{ranu2013mining}, and least square functions~\cite{Dang2016icdm}; and topological constraints, such as the types of subgraphs aforementioned~\cite{perozzi2016scalable, gupta2014local, neill2013fast, perozzi2014focused, gunnemann2014gamer}, compact subgraphs~\cite{rozenshtein2014event}, trees~\cite{mairal2011convex}, and paths~\cite{asteris2015stay}.
 2) \textbf{Lack of good tradeoff between tractability and quality guarantees.} The methods for detecting \underline{\smash{anomalous connected subgraphs}} conduct exhaust search over all feasible subgraphs (neighborhoods), but will be intractable when the number of feasible subgraphs is large (e.g., all connected subgraphs). Several methods for detecting \underline{\smash{coherent dense subspace}} \underline{\smash{clusters}} are tractable to large networks, but do not provide worst-case theoretical guarantees on the quality of the detected clusters. 
 
This paper presents a novel generic and theoretical framework to address the above two main limitations of existing methods for a broad class of interesting subspace cluster detection problems. In particular, we consider the general form of subspace cluster detection as an optimization problem that has a general score function measuring the interestingness of a subset of features and a cluster of nodes, a sparsity constraint on the subset of features, and topological constraints on the cluster of nodes.  We propose a novel subspace graph-structured matching pursuit algorithm, namely, \texttt{SG-Pursuit}, to approximately solve this general problem in nearly-linear time. The key idea is to iteratively search for a close-to-optimal solution by solving easier subproblems in each iteration, including i) identification of topological-free clusters of nodes and a sparsity-free subset of attributes that maximizes the score function in a sub-solution-space determined by the gradient of the current solution; and ii) projection of the identified intermediate solution onto the solution-space defined by the sparsity and topological constraints. The contributions of this work are summarized as follows:
\begin{itemize}
\item \textbf{Design of a generic and efficient approximation algorithm for the subspace cluster detection problem.} We propose a novel generic algorithm, namely, \texttt{SG-Pursuit}, to approximately solve a broad class of subspace cluster detection problems that are defined by different score functions and topological constraints in nearly-linear time. 
To the best of our knowledge, this is the first-known generic algorithm for such problems. 
\item \textbf{Theoretical guarantees and connections.} We present a theoretical analysis of the proposed \texttt{SG-Pursuit} and show that \texttt{SG-Pursuit} enjoys a geometric rate of convergence and a tight error bound on the quality of the detected subspace clusters. 
We further demonstrate that \texttt{SG-Pursuit} enjoys strong guarantees analogous to state-of-the-art methods for sparse feature selection in high-dimensional data and for subgraph detection in attributed networks.   
\item \textbf{Compressive experiments to validate the effectiveness and efficiency of the proposed techniques.} \texttt{SG-Pursuit} was specialized to conduct the specific tasks of coherent dense subspace cluster detection and anomalous connected subspace cluster detection on several real-world data sets. The results demonstrate that \texttt{SG-Pursuit} outperforms state-of-the-art methods that are designed specifically for these tasks, even that \texttt{SG-Pursuit} is designed to address general subspace cluster detection problems.  
\end{itemize}

\noindent \textbf{Reproducibility}: The implementations of \texttt{SG-Pursuit} and baseline methods and the data sets are available via the link~\cite{fullversion}. 

The rest of this paper is organized as follows. Section~\ref{sect:method} introduces the proposed method \texttt{SG-Pursuit} and analyzes its theoretical properties. Section~\ref{sect:applications} discussions applications of our proposed algorithm for the tasks of coherent dense subspace cluster detection and anomalous connected subspace cluster detection. Experiments on several real world benchmark datasets are presented in Section~\ref{sect:experiments}. Section~\ref{sect:conclusions} concludes the paper and describes future work. 

\section{METHOD \texttt{SG-Pursuit}}
\label{sect:method}

In this section we first introduce the notation and define the problem of subspace cluster detection formally. Next, we present the algorithm \texttt{SG-Pursuit} and analyze its theoretical properties, including its convergence rate, error-bound, and time complexity. 

\subsection{Problem Formulation}
\label{sect:method-problem-formulation}
We consider a multi-attributed network that is defined as $\mathbb{G} = (\mathbb{V}, \mathbb{E}, w)$, where $\mathbb{V} = \{1, \cdots, n\}$ is the ground set of nodes of size $n$, $\mathbb{E} \subseteq \mathbb{V} \times \mathbb{V}$ is the set of edges, and the  function $w:\mathbb{V}\rightarrow \mathbb{R}^p$ defines a vector of attributes of size $p$ for each node $v\in \mathbb{V}$: $w(v) \in \mathbb{R}^p$. For simplicity, we denote the attribute vector $w(v)$ by $w_v$. 

We introduce two vectors of coefficients, including $x \in \mathbb{R}^n$ and $y\in \mathbb{R}^p$, that will be optimized for detecting the most interesting subspace cluster in $\mathbb{G}$, where $x$ identifies the cluster (subset) of nodes and $y$ identifies their relevant attributes.  In particular, the vector $x$ refers to the vector of coefficients of the nodes in $\mathbb{V}$. Each node $i \in \mathbb{V}$ has a coefficient score $x_i$ indicating the importance of this node in the cluster of interest. If $x_i \neq 0$, it means that the node $i$ belongs to the cluster of interest. Similarly, the vector $y$ refers to the vector of coefficients of the $p$ attributes. Each attribute  $j \in \{1, \cdots, p\}$ has a coefficient score $y_j$ indicating the relevance of this attribute to the clusters of interest. Let $\text{supp}(x)$ be the support set of indices of nonzero entries in $x$: $\text{supp}(x) = \{i \ | \ x_i \neq 0\}$. Then the support set $\text{supp}(x)$ represents the subset of nodes that belong to the cluster of interest. The support set $\text{supp}(y)$ represents the subset of relevant attributes. 
We define the feasible space of clusters of nodes as 
\begin{eqnarray}\mathbb{M}(k) =\{ S\ | \ S\subseteq \mathbb{V}; |S| \le k; \mathbb{G}_{S} \text{ satisfies predefined topological} \nonumber \\ \text{constraints. } \},\nonumber 
\end{eqnarray}
where $S$ refers to a subset of nodes in $\mathbb{V}$, $\mathbb{G}_{S} = (S, \mathbb{E}\cap S\times S)$ refers to the subgraph induced by $S$, $|S|$ refers to the total number of nodes in $S$, and $k$ refers to an upper bound on the size of the cluster. The topological constraints can be any topological constraints on $\mathbb{G}_S$, such as connected subgraphs~\cite{perozzi2016scalable, neill2013fast}, dense subgraphs~\cite{gunnemann2014gamer, perozzi2014focused}, subgraphs that are isomorphic to a query graph~\cite{gupta2014local}, compact subgraphs~\cite{rozenshtein2014event}, trees~\cite{mairal2011convex}, and paths~\cite{asteris2015stay}, among others.

Based on the above notations, we consider \textbf{a general form of the subspace cluster detection problem} as 
\begin{eqnarray}
\max_{x \in \mathcal{C}_x, y\in \mathcal{C}_y} f(x, y)\ \ \ s.t.\ \ \ \text{supp}(x) \in \mathbb{M}(k) \text{  and  } \|y\|_0 \le s, \label{prob1}
\end{eqnarray}
where $f(x, y): \mathbb{R}^n \times \mathbb{R}^p \rightarrow \mathbb{R}$ is a score function that measures the overall level of interestingness of the subspace clusters indicated by $x$ and $y$;  $\mathcal{C}_x \subseteq \mathbb{R}^n$ represents a convex set in the Euclidean space $\mathbb{R}^n$, $\mathcal{C}_y \subseteq \mathbb{R}^p$ represents a convex set in the Euclidean space $\mathbb{R}^p$, $\mathbb{M}(k)$ refers to the feasible space of clusters of nodes as defined a above, and $s$ refers to an upper bound on the number of attributes relevant to the subspace clusters of interest. The parameters $k$ and $s$ are predefined by the user. Let $\hat{x}$ and $\hat{y}$ be the solution to Problem~(\ref{prob1}). Denote by $S$ the support set $\text{supp}(\hat{x})$ that represents the most interesting cluster of nodes, and by $R$ the support set $\text{supp}(\hat{y})$ represents the subset of relevant attributes. The most interesting subspace cluster can then be identified as $(S, R)$. 

\begin{figure}[t]
\center
\includegraphics[width=8.1cm]{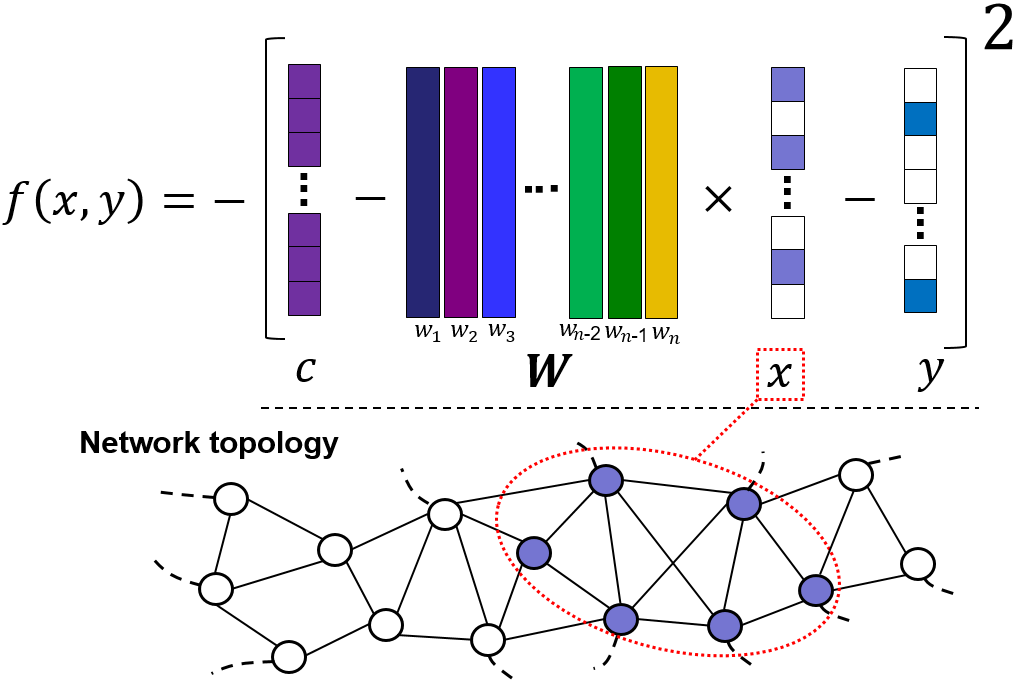}
\captionof{figure}{An example function of $f(x, y)$ (negative squared error function~(\ref{eqn:squared-error})) for robust linear regression models that has been widely used in anomaly detection tasks~\cite{Dang2016icdm, rousseeuw2005robust, tong2011non, shcherbatyi2016convexification}. In this example, the vector $x$ is a vector of sparse coefficients of the nodes in the input network that must satisfy the topological constraints ($\mathbb{M}(k = 6)$): the size of $\text{supp}(x)$ is at most $6$. The residual vector $y$ is a sparse vector as defined by the constraint $\|y\|_0 \le s$ and is used to identify anomalous attributes. }\label{fig:problem-formulation}
\vspace{-10pt}
  \end{figure}
  
As illustrated in Figure~\ref{fig:problem-formulation}, an example score function $f(x, y)$ is a negative squared error function for robust linear regression that has been widely used in anomaly detection tasks~\cite{Dang2016icdm, rousseeuw2005robust, tong2011non, shcherbatyi2016convexification}: 
\begin{eqnarray}
f(x, y) = - \|c - W^\intercal x - y\|_2^2, \label{eqn:squared-error}
\end{eqnarray}
where $x \in C_x := {\mathbb{R}}^n$, $y \in C_y :={\mathbb{R}}^p$, $c \in \mathbb{R}^p$ refers to a vector of observed response values, and $W = [w_1, w_2, \cdots, w_n]^\intercal \in \mathbb{R}^{n \times p}$. The residual vector $y$ is used to identify anomalous attributes, and its sparsity $s$ is usually much smaller than $p$ (the total number of attributes).  There are also applications where both $x$ and $y$ need to be vectors of positive coefficients~\cite{tong2011non}: $C_x := {\mathbb{R}^+}^n$ and  $C_y :={\mathbb{R}^+}^p$. 

\begin{remark}
There are scenarios where $x$ is considered as a vector of binary values, instead of numerical coefficients, and the resulting problem becomes a discrete optimization problem that is NP-hard in general and does not have known solutions. 
In this case, by relaxing the input domain of $x$ from $\{0, 1\}^n$ to the convex set $\mathcal{C}_x := [0, 1]^n$ and replacing the score function $f(x,y)$ with its tight concave surrogate function, the resulting relaxed problem becomes a special case of Problem~(\ref{prob1}). 
In particular, when the cost function is a supermodular function of $x$, a tight concave surrogate function can be obtained based on Lobasz extensions, such that the solutions to the relaxed problem are identical to the solutions to the original discrete optimization problem. In addition, the same equivalence also holds for a number of popular non-convex functions that are non-supermodular, such as Hinge and Squared Hinge functions, and their tight concave surrogate functions have been studied in recent work~\cite{shcherbatyi2016convexification, ballerstein2013convex}. 

\end{remark}

\begin{remark}

Problem (\ref{prob1}) considers the detection of the most interesting subspace cluster in a multi-attributed network. There are applications, where top $k$ most interesting subspace clusters are of interest, where $k$ is predefined by the user.  
In this case, the $k$ clusters can be identified one-by-one, repeatedly, by solving Problem (\ref{prob1}) for each subspace cluster and deflating the attribute data to remove information captured by previously extracted subspace clusters. 
\end{remark}

\subsection{Head and Tail Projections on $\mathbb{M}(k)$}

Before we present our proposed algorithm \texttt{SG-Pursuit}, we first introduce two major components related to the support of the topological constraints ``$\text{supp}(x) \in \mathbb{M}(k)$'', including head and tail projections. The key idea is that, suppose we are able to find a good intermediate solution $x$ that does not satisfy this constraint, these two types of projections will be used to find good approximations of $x$ in the feasible space defined by $\mathbb{M}(k)$. 
\begin{itemize}
\item \textbf{Tail Projection} ($\text{T}(x)$)\cite{hegde2015nearly}: Find a $S \subseteq \mathbb{V}$ such that 
\begin{eqnarray}
\|x - x_S\|_2 \le c_T \cdot \min_{S^\prime \in \mathbb{M}(k)} \|x - x_{S^\prime}\|_2,
\end{eqnarray}
where $c_T\ge 1$, and $x_S$ is the restriction of $x$ to indices in $S$: we have $(x_S)_i = x_i$ for $i \in S$ and $(x_S)_i = 0$ otherwise. When $c_T = 1$, $\text{T}(x)$ returns an optimal solution to the problem: $\min_{S^\prime \in \mathbb{M}(k)} \|x - x_{S^\prime}\|_2$. When $c_T > 1$, $\text{T}(x)$ returns an approximate solution to this problem with the approximation factor $c_T$. 
\item \textbf{Head projection} ($\text{H}(x)$)\cite{hegde2015nearly}: Find a $S \subseteq \mathbb{V}$ such that 
\begin{eqnarray}
\|x_S\|_2 \ge c_H\cdot \max_{S^\prime \in \mathbb{M}(k)} \|x_{S^\prime}\|_2,\label{head-projection}
\end{eqnarray}
where $c_H \leq 1$. When $c_H = 1$, $\text{H}(x)$ returns an optimal solution to the problem: $\max_{S^\prime \in \mathbb{M}(k)} \|x_{S^\prime}\|_2$. When $c_H < 1$, $\text{H}(x)$ returns an approximate solution to this problem with the approximation factor $c_H$. 
\end{itemize}

It can be readily proved that, when $c_T = 1$ and $c_H = 1$,  both $\text{T}(x)$ and $\text{H}(x)$ return the same subset $S$, and the corresponding vector $x_S$ is an optimal solution to the standard projection oracle in the traditional projected gradient descent algorithm~\cite{bahmani2016learning}: 
\begin{eqnarray}
\arg \min_{x^\prime \in \mathbb{R}^n} \|x - x^\prime\|_2 \ \ s.t.\ \  \text{supp}(x^\prime) \in \mathbb{M}(k), \label{projection-oracle} \end{eqnarray}
which is NP-hard in general for popular topological constraints, such as connected subgraphs and dense subgraphs~\cite{qian2014connected}. However, when $c_T> 1$ and $c_H$ < 1, $\text{T}(x)$ and $\text{H}(x)$ return different approximate solutions to the standard projection problem~(\ref{projection-oracle}). Although the head and tail projections are NP-hard problems when $c_T=1$ and $c_H = 1$, these two projections can often be implemented in nearly-linear time when we allow relaxations on $c_T$ and $c_H$: $c_T > 1$ and $c_H < 1$. For example, when the topological constraints considered in $\mathbb{M}(k)$ is that: ``$\mathbb{G}_{S_i}$ is a connected subgraph'', where $S_i$ is a specific cluster of nodes, the resulting head and tail projections can be implemented in nearly-linear time with the parameters: $c_T = \sqrt{7}$ and $c_H = \sqrt{1/14}$~\cite{hegde2015nearly}. The impact of these two parameters on the performance of \texttt{SG-Pursuit} will be discussed in Section~\ref{section:theoretical-analysis}.  

As discussed above, the head and tail projections can be considered as two different approximations to the standard projection problem~(\ref{projection-oracle}). It has been demonstrated that the joint utilization of both head and tail projections is critical in design of approximate algorithms for network-related optimization problems~\cite{hegde2015nearly, hegde2015approximation, chen2016generalized, zhou2016icdm}. 

\subsection{Algorithm Details}
We propose a novel \texttt{S}ubspace \texttt{G}raph-structured matching \texttt{Pursuit} algorithm, namely, \texttt{SG-Pursuit}, to approximately solve Problem (\ref{prob1}) in nearly-linear time. The key idea is to iteratively search for a close-to-optimal solution by solving easier subproblems in each iteration $i$, including i) \underline{\smash{identification}} of the intermediate solution ($b^i_x$, $b^i_y$) that maximizes the score function $f(x, y)$ in a solution-subspace determined by the partial derivatives of the function on the current solution, including $\nabla_x f(x^{i}, y^{i})$ and $\nabla_y f(x^{i}, y^{i})$, and ii) \underline{\smash{projection}} of the  intermediate solution ($b^i_x$, $b^i_y$) to the feasible space defined by the topological constraints: ``$\text{supp}(x) \in \mathbb{M}(k)$'', and the sparsity constraint: ``$\|y\|_0 \le s$''. The projected solution ($x^{i+1}$, $y^{i+1}$) is then the updated intermediate solution returned by this iteration. 

The main steps of \texttt{SG-Pursuit} are shown in Algorithm~\ref{alg:SG-Pursuit}. The procedure generates a sequence of intermediate solutions $(x^0, y^0)$, $(x^1, y^1), \cdots$, from an initial solution $(x^0, y^0)$. At the $i$-th iteration, the first step (\textbf{Line 6}) calculates the partial derivative $\nabla_x f(x^{i}, y^{i})$, and then identifies a subset of nodes via head projection that returns a support set with the head value at least a constraint factor $c_H$ of the optimal head value: ``$\Gamma_x = \text{H}(\nabla_x f(x^{i}, y^{i}))$''. The support set $\Gamma_x$ can be interpreted as the directions  where the nonconvex set ``$\text{supp}(x) \in \mathbb{M}(k)$'' is located, within which pursuing the maximization over $y$ will be most effective. The second step (\textbf{Line 7}) identifies the $2s$ nodes of the partial derivative vector $\nabla_y f(x^{i}, y^{i})$ that have the largest magnitude that are chosen as the directions in which pursuing the maximization on $y$ will be most effective: \[\Gamma_y\ = \argmax_{R \subseteq \{1, \cdots, p\}}\{\|[\nabla_y f(x^i, y^i)]_R\|_2^2: \|R\|_0 \le 2 s\},\] where $[\nabla_y f(x^i, y^i)]_R$ refers to the projected vector in the subspace defined by the subset $R$. Denote by $w$ the projected vector $[\nabla_y f(x^i, y^i)]_R$. We then have $w_i = [\nabla_y f(x^i, y^i)]_i$, the $i$-th entry in the gradient vector $\nabla_y f(x^i, y^i)$, if $i \in R$; otherwise, $w_i = 0$. The subsets $\Gamma_x$ and $\Gamma_y$ are then merged in \textbf{Line 8} and \textbf{Line 9} with the supports of the current estimates ``$\text{supp}(x^i)$'' and ``$\text{supp}(y^i)$'', respectively, to obtain ``$\Omega_x= \Gamma_x\cup \text{supp}(x^i)$'' and ``$\Omega_y= \Gamma_y\cup \text{supp}(y^i)$''. The combined support sets define a subspace of $x$ and $y$ over which the function $f(x,y)$ is maximized to produce an intermediate solution in \textbf{Line 10}: 
\[(b_x^i, b_y^i) = \argmax_{x \in \mathcal{C}_x, y\in \mathcal{C}_y} f(x, y) \ \ s.t. \ \ \text{supp}(x) \subseteq \Omega_x, \text{supp}(y) \subseteq \Omega_y. \]
Then a subset of nodes are identified via tail projection of $b_x^i$ in \textbf{Line 11}: ``$\Psi_x^{i+1} = \text{T}(b_x^i)$'', that returns a support set with the tail value at most a constant $c_T$ times larger than the optimal tail value. A subset of attributes of size $s$ that have the largest magnitude are chosen in \textbf{Line 12} as the subset of relevant attributes: \[\Psi_y^{i+1} = \argmax_{R \subseteq \{1, \cdots, p\}}\{\|[b_y^i]_R\|_2^2: \|R\|_0 \le  s\}.\] As the final steps of this iteration (\textbf{Line 13} and \textbf{Line 14}), the estimates $x^{i+1}$ and $y^{i+1}$ are updated as the restrictions of $b_x^i$ and $b_y^i$ on the support sets $\Psi_x^{i+1}$ and $\Psi_y^{i+1}$, respectively: ``$x^{i+1}\ = [b_x^i]_{\Psi_x^{i+1}}$'' and ``$y^{i+1}\ = [b_y^i]_{\Psi_y^{i+1}}$.'' These steps are conducted to ensure that the estimates $x^{i+1}$ and $y^{i+1}$ returned by each iteration always satisfy the sparsity and topological constraints, respectively. After the termination of the iterations, \textbf{Line 17} identifies the subspace cluster: $\mathcal{C}=(\Psi_x^{i}, \Psi_y^{i})$, where $\Psi_x^{i}$ represents the subset (cluster) of nodes and $\Psi_y^{i}$ represents the subset of relevant attributes.

      \begin{algorithm}[H]
        \caption{\texttt{SG-Pursuit}}
        \begin{algorithmic}[1]
         \STATE \textbf{Input}: Network instance $\mathbb{G}$ and the parameters, including $k$ (maximum number of nodes in the subspace cluster) and $s$ (maximum size of selected features). 
         \STATE \textbf{Output}: The vectors of coefficients of nodes and attributes, including $x^i$ and $y^i$, and the identified subspace cluster $\mathcal{C}$. 
         \STATE $\epsilon = 0.0001$ \ \%  The termination criterium of the iterations
          \STATE $i = 0$; $x^i, y^i = \text{initial vectors}$\;
          \REPEAT 
          \STATE $\Gamma_x\ = \text{H}(\nabla_x f(x^i, y^i))$\; 
          \STATE $\Gamma_y\ = \arg\max_{R \subseteq \{1, \cdots, p\}}\{\|[\nabla_y f(x^i, y^i)]_R\|_2^2: \|R\|_0 \le 2 s\}$\;
          \STATE $\Omega_x= \Gamma_x\cup \text{supp}(x^i)$\; 
          \STATE$\Omega_y= \Gamma_y\cup \text{supp}(y^i)$\;
          \STATE $(b_x^i, b_y^i) = \argmax_{x \in \mathcal{C}_x, y\in \mathcal{C}_y} f(x, y)$ \ \   s.t. \ \  $\text{supp}(x) \subseteq \Omega_x$, $\text{supp}(y) \subseteq \Omega_y$\; 
          \STATE $\Psi_x^{i+1} = \text{T}(b_x^i)$\; 
          \STATE $\Psi_y^{i+1} = \arg\max_{R \subseteq \{1, \cdots, p\}}\{\|[b_y^i]_R\|_2^2: \|R\|_0 \le  s\}$\;
          \STATE $x^{i+1}\ = [b_x^i]_{\Psi_x^{i+1}}$
          \STATE $y^{i+1}\ = [b_y^i]_{\Psi_y^{i+1}}$
          \STATE $i\ \  \ \ \ \ \ = i+1$\;
          \UNTIL{$\|x^{i} - x^{i-1}\|\le \epsilon $ and $\|y^{i} - y^{i-1}\|\le \epsilon $}
          \STATE $\mathcal{C}=(\Psi_x^{i}, \Psi_y^{i})$.
          \STATE \textbf{return}  $x^{i}, y^{i}, \mathcal{C}$     
           \end{algorithmic}\label{alg:SG-Pursuit}
      \end{algorithm}
      
\subsection{Theoretical Analysis}
\label{section:theoretical-analysis}
In order to demonstrate the accuracy and efficiency of \texttt{SG-Pursuit}, we require that the score function $f(x, y)$ satisfies the Restricted Strong Concavity/Smoothness (RSC/RSS) condition as follows: 
\begin{definition}[Restricted Strong Concavity/Smoothness (RSC/RSS)]
A score function $f$ satisfies the $(\mathbb{M}(k), s, \gamma^-, \gamma^+)$-$RSS/RSC$ if, for every $x, x^\prime \in \mathbb{R}^n$ and $y, y^\prime \in \mathbb{R}^p$ with $\text{supp}(x) \subseteq \mathbb{M}(2k)$, $\text{supp}(x^\prime) \subseteq \mathbb{M}(2k)$, $|\text{supp}(y)| \leq 2s$, and $|\text{supp}(y^\prime)| \leq 2s$, the following inequalities hold:
\begin{small}
\begin{eqnarray}
\frac{\gamma^-}{2}\left(\|x - x^\prime\|_2^2 + \|y - y^\prime\|_2^2\right) \le \nonumber \\  f(x, y) - f(x^\prime, y^\prime)  -   \nabla_x f(x, y)^\intercal(x - x^\prime) - \nabla_y f(x, y)^\intercal(y - y^\prime) \le \nonumber \\ \frac{\gamma^+}{2}\left(\|x - x^\prime\|_2^2 + \|y - y^\prime\|_2^2\right).
\label{RSS-RSC}
\end{eqnarray}
\end{small}\label{def:RSS-RSC}
\end{definition}
The RSC/RSS condition basically characterizes cost functions
that have quadratic bounds on the derivative of the objective
function when restricted to the graph-structured vector $x$ and the sparsity-constrained vector $y$. When the score function $f$ is a quadratic function of $x$ and $y$, RSC/RSS condition degeneralizes to the restricted isometry property (RIP) that is well-known in the field of compressive sensing. For example, we consider the negative squared error function~(\ref{eqn:squared-error}) as discussed in  Section~\ref{sect:method-problem-formulation}: $f(x, y) = - \|c - W^\intercal x - y\|_2^2$. Let $\bar{W} =[W^\intercal, I]$, where $I$ is a $p$ by $p$ identity matrix. Let $z = [x^\intercal, y^\intercal]^\intercal$. 
The RSC/RSS condition can then be reformulated as the RIP condition: 
\[
(1 - \delta)\|z\|_2^2 \le \|\bar{W} z\|_2^2 \le (1 + \delta) \|z\|_2^2,
\]
where $\gamma^+ = 2(1 + \delta)$, $\gamma^- =2 (1 - \delta)$, and $\delta\in [0, 1]$ is the standard parameter as defined in RIP. However, the RIP condition in this example \underline{\smash{is different from}} the traditional RIP condition in that the components of $z$, including $x$ and $y$, must satisfy the constraints related to $\mathbb{M}(k)$ and the sparsity $s$ as described in Definition~\ref{def:RSS-RSC}.

\begin{theorem}
If the score function $f$ satisfies the property $(\mathbb{M}(k)$, $s, \gamma^-, \gamma^+)$-$RSS/RSC$, then for any true $(x^*, y^*) \in \mathbb{R}^n\times \mathbb{R}^p$, the iterations of the proposed algorithm \texttt{SG-Pursuit} satisfy the inequality
\begin{eqnarray}
\|r_x^{i+1}\|_2 + \|r_y^{i+1}\|_2 \le \alpha \Big(\|r_x^{i}\|_2 + \|r_y^{i}\|_2 \Big) +  \beta (\varepsilon_x + \varepsilon_y)\nonumber
\end{eqnarray}
where $r_x^{i+1} = x^{i+1} - x^*$, $r_y^{i+1} = y^{i+1} - y^*$,
$\alpha_0 = c_\text{H}(1-\rho) - \rho$, $\rho = \sqrt{1 - (\frac{\gamma^-}{\gamma^+})}$, $\beta_0  = (c_H+1) \frac{\gamma^-}{(\gamma^+)^2}$,  $\alpha = \frac{(c_T + 1) \sqrt{2 - 2\alpha_0^2} }{1 -  \sqrt{2}\rho}$,  $\beta = \frac{(c_T + 1) }{1 -  \sqrt{2}\rho}\Big(\frac{\gamma^-}{(\gamma^+)^2} +\Big(\frac{\sqrt{2}\alpha_0 \beta_0  }{1 - \alpha_0^2} + \frac{\sqrt{2}\beta_0}{\alpha_0}\Big)  \Big)$, $\varepsilon_x = \max_{S \in \mathbb{M}(2k)} \|[\nabla f_x(x^*, y^*)]_S\|_2^2$, and $\varepsilon_y = \max_{R\subseteq\{1, \cdots, p\}, |R| \le 3s } \|[\nabla f_y(x^*, y^*)]_R\|_2^2$. \label{theorem:accuracy}
\end{theorem}
\begin{proof}
See the Appendix~\ref{proof-theorem-2-2} for details. 
\end{proof}

\begin{theorem}
Let $(x^*, y^*)$ the optimal solution to Problem~(\ref{prob1}) and $f$ be a score function that satisfies the $(\mathbb{M}(k), s, \gamma^-, \gamma^+)$-$RSS/RSC$ property. Let $T$ and $H$ be the tail and head projections with $c_T$ and $c_H$ such that $0 < \alpha < 1$. Then after $t = \left \lceil \log \left(\frac{\|x^*\|_2 + \|y^*\|_2}{\varepsilon_x + \varepsilon_y}\right) / \log \frac{1}{\alpha} \right \rceil $ iterations, \texttt{SG-Pursuit} returns a single estimate $(\hat{x}, \hat{y})$ satisfying 
\begin{eqnarray}\|\hat{x} - x^*\|_2 + \|\hat{y} - y^*\|_2 \le c(\varepsilon_x + \varepsilon_y),\label{error-bound}\end{eqnarray}
where  $c = (1 + \frac{\beta}{1-\alpha})$ is a fixed constant. Moreover, \texttt{SG-Pursuit} runs in time 
\begin{eqnarray}
O\left((T_1 + T_2 + p\log p) \log \Big((\|x^*\|_2 + \|y^*\|_2)/(\varepsilon_x + \varepsilon_y)\Big)\right), \label{time-complexity1}
\end{eqnarray}
where $T_1$ is the time complexity of one execution of the subproblem in Line 10 in \texttt{SG-Pursuit} and $T_2$ is the time complexity of one execution of the head and tail projections. 

In particular, when the connectivity constraint or a density constraint  is considered as the topological constraint on the feasible clusters of nodes in $\mathbb{M}(k)$, there exist efficient algorithms for the head and tail projections that have the time complexity $O(|\mathbb{E}| \log^3 n)$~\cite{hegde2015nearly}. When $s$ and $k$ are fixed small constants with respect to $n$, the subproblem in Line 10 in \texttt{SG-Pursuit} can be solved in nearly linear time  in practice using convex optimization algorithms, such as the project gradient descent algorithm. Therefore, under these conditions, for coherent dense subgraph detection and connected anomalous subspace cluster detection problems, \texttt{SG-Pursuit} has a nearly-linear time complexity on the network size $n$ and the cardinality of attributes $p$:  \begin{eqnarray}O\left((|\mathbb{E}| \log^3 n + p\log p) \log \left((\|x^*\|_2 + \|y^*\|_2)/(\varepsilon_x + \varepsilon_y)\right)\right).  \label{time-complexity2}\end{eqnarray}\label{theorem:convergence}
  
\end{theorem}
\begin{proof}
From Theorem~\ref{theorem:accuracy}, the following inequality can be obtained via an inductive argument:
\[\|x^i - x^*\|_2 + \|y^i - y^*\|_2 \le \alpha^i (\|x^*\|_2 + \|y^*\|_2) +  \beta (\varepsilon_x + \varepsilon_y)\sum_{j=0}^i \alpha^i.\]
For $i = \left \lceil \log \left(\frac{\|x^*\|_2 + \|y^*\|_2}{\varepsilon_x + \varepsilon_y}\right) / \log \frac{1}{\alpha} \right \rceil $, we have $\alpha^i (\|x^*\|_2 + \|y^*\|_2) \le (\varepsilon_x + \varepsilon_y)$. The geometric series $\sum_{j=0}^i \alpha^i$ can be bounded by $\frac{1}{1 - \alpha}$. The error bound~(\ref{error-bound}) can be obtained by coming the preceding inequalities. The time complexity of the subproblem in Line 10 is denoted by $O(T_1)$, and the \underline{\smash{time complexities}} of both head and tail projections are denoted by $O(T_2)$.  The time complexity to solve the subproblem in Line 7 is $O(p\log p)$, as the exact solution can be obtained by sorting the entries in $\nabla_y f(x^i, y^i)$ in a descending order based their absolute values, and then returning the indices of the top $2s$ entries. Similarly, the time complexity to solve the subproblem in Line 12 is $O(p\log p)$. As the total number of iterations is $O\log (\frac{\|x^*\|_2 + \|y^*\|_2}{\varepsilon_x + \varepsilon_y})$, the time complexity specified in Equation~(\ref{time-complexity1}) can be calculated. accordingly. When $T_1$ is bounded by $O(n \log n)$ and $T_2 = |\mathbb{E}| \log^3 n$, the nearly-linear time complexity specified in Equation~(\ref{time-complexity2}) can be obtained. 
\end{proof} 

Theorem~\ref{theorem:convergence} shows that \texttt{SG-Pursuit} enjoys  \underline{\smash{a geometric rate of}}  \underline{\smash{convergence}} and the estimation error is determined by the multiplier of $(\varepsilon_x + \varepsilon_y)$, where $\varepsilon_x = \max_{S \in \mathbb{M}(8k)} \|[\nabla f_x(x^*, y^*)]_S\|_2^2$, and $\varepsilon_y = \max_{R\subseteq\{1, \cdots, p\}, |R| \le 2s } \|[\nabla f_y(x^*, y^*)]_R\|_2^2$. The shrinkage rate $\alpha < 1$ controls the converge rate of $\texttt{SG-Pursuit}$. In particular, if the true $x^*$ and $y^*$ are sufficiently close to an unconstrained maximum of $f$, then the estimation error is negligible because both $\varepsilon_x$ and $\varepsilon_y$ have small magnitudes. Especially, in the idea case where $\varepsilon_x = \varepsilon_y = 0$, it is guaranteed that we can obtain the true $x$ and $y$ to arbitrary precisions. Note that we can make $\gamma^+$ and $\gamma^-$ as close as we desire, such that $\gamma^+/ \gamma^- \approx 1$, since this assumption only affects the measurement bound by a constraint factor. In this case, in order to ensure that $\alpha \approx (c_T + 1)\sqrt{2 - 2 c_H^\intercal}< 1$, the factors $c_T$ and $c_H$ should satisfy the inequality: $c_H^2 > 1 - 1/\left(2(1+c_T)^2\right)$. As proved in~\cite{hegde2015approximation}, the fixed factor $c_H$ of any given algorithm for the head projection can be boosted to any arbitrary constant $c_H^\prime < 1$, such that the above condition can be satisfied. This indicates the flexibility of designing approximate algorithms for head and tail projections in order to ensure the geometric convergence rate of \texttt{SG-Pursuit}. 

\begin{remark}(Connections to existing methods) \texttt{SG-Pursuit} is a generalization of the \texttt{GraSP} (Gradient Support Pursuit) method~\cite{bahmani2013greedy}, a state-of-the-art method for general sparsity-constrained optimization problems, and the \texttt{Graph-MP} method~\cite{chen2016generalized}, a state-of-the-art method for general graph-structured sparse optimization problems. In particular, when we fix $x$ and only update $y$ in the steps of \texttt{SG-Pursuit}, \texttt{SG-Pursuit} then degeneralizes to \texttt{GraSP}. When we fix $y$ and only update $x$ in the steps of \texttt{SG-Pursuit}, \texttt{SG-Pursuit} then degeneralizes to \texttt{Graph-MP}. Surprisingly, even that \texttt{SG-Pursuit} concurrently optimizes $x$ and $y$, its convergence rate is of  \underline{\smash{the same order}} as those of \texttt{Graph-MP} and \texttt{GraSP} under the RSC/RSS property.

\end{remark}
 
\section{Example Applications}
\label{sect:applications}
In this section, we specialize \texttt{SG-Pursuit} to address two typical subspace cluster detection problems in multi-attributed networks, including coherent dense subspace cluster detection and anomalous connected subspace cluster detection. The former searches for subsets of nodes that show high similarity in subsets of their attributes and that are as well densely connected within the input network. The \underline{\smash{coherence score}} function, as shown in Table~\ref{table:scoreFunctions}, is defined as the log likelihood ratio function, $\log \frac{\text{Prob}\left(\text{Data}|H_1(x,y)\right)}{\text{Prob}(\text{Data}|H_0)}$, that corresponds to the hypothesis testing framework: 
\begin{itemize}
\item Under the null ($H_0$), $w_{i, j} \sim \mathcal{N}(0, 1), \forall i \in \mathbb{V}, j \in \{1, \cdots, p\}$, where $w_{i, j}$ refers to the observed value of the $j$-th attribute of node $i$;
\item Under the alternative $H_1(x, y)$, $w_{i, j} \sim \mathcal{N}(\mu_j, 1)$, if $x_i = 1$ and $y_j = 1$; otherwise, $w_{i, j} \sim \mathcal{N}(0, \sigma)$, where $x\in \{0, 1\}^n$, $y\in \{0, 1\}^p$, and $x_i = 1$ indicates that node $i$  belongs to the cluster, $y_j = 1$ indicates that the attribute $j$ belongs to the subset of coherent attributes. Each coherent attribute $j$ has a different mean parameter $\mu_j$ and the variance $\sigma$ should be less than 1, the variance of an incoherent attribute, in order to ensure the coherence of its observations. $\sigma$ is set 0.01 by default. 
\end{itemize}
The latter (anomalous connected subspace cluster detection) searches for subsets of nodes that are significantly different from the other nodes on subsets of their attributes and that are as well connected within the input network. The \underline{\smash{elevated mean scan statistic}}, as shown in Table~\ref{table:scoreFunctions}, is defined as the log likelihood ratio function that corresponds to a hypothesis testing framework that is the same as the above, except that 1) ``coherent'' is replace by ``anomalous'', 2) the mean of each anomalous attribute is greater than (or more anomalous than) 0, the mean of a normal attribute, and the standard deviation of each anomalous attribute is set to 1 (the same variance of a normal attribute). The \underline{\smash{Fisher test statistic}} function is considered when each $w_{i,j}$ represents the level of anomalous (e.g., negative log p-value) of the $j$-th attribute of node $i$, and $x^\intercal W y$ represents the overall level of anomalous of $x$ and $y$. A large class of scan statistic functions for anomaly detection can be transformed to the fisher test statistic function using a 2-step procedure as proposed in~\cite{qian2014connected}. 
The \underline{\smash{negative squared error}} is considered as the score function of anomalous subspace cluster detection in a regression setting and is introduced in Section~\ref{sect:method-problem-formulation}.  

\begin{theorem}
When the attribute matrix $W$ satisfies certain properties, the score functions, including the elevated mean scan statistic, the Fisher's test statistic, the negative square error, and the logistic function, satisfy the RSC/RSS property as described in Definition~\ref{def:RSS-RSC}.  \label{theorem-score-functions}
\end{theorem}
\begin{proof}
See the Appendix~\ref{proof-theorem-3.1} for details about the specific properties of $W$ and parameters ($\gamma^+$ and $\gamma^-$) of the RSC/RSS property for each score function.  
\end{proof}
Theorem~\ref{theorem-score-functions} demonstrates that the theoretical guarantees of \texttt{SG-Pursuit} as analyzed in Section~\ref{section:theoretical-analysis} are applicable to a number of popular score functions for subspace cluster detection problems. We note that \texttt{SG-Pursuit} also performs well in practice on the score functions not satisfying the RSC/RSS property as demonstrated in Section~\ref{sect:experiments} using the coherence score function as shown in Table~\ref{table:scoreFunctions}. 

 \begin{figure*}[t]
\center
\includegraphics[width=\textwidth]{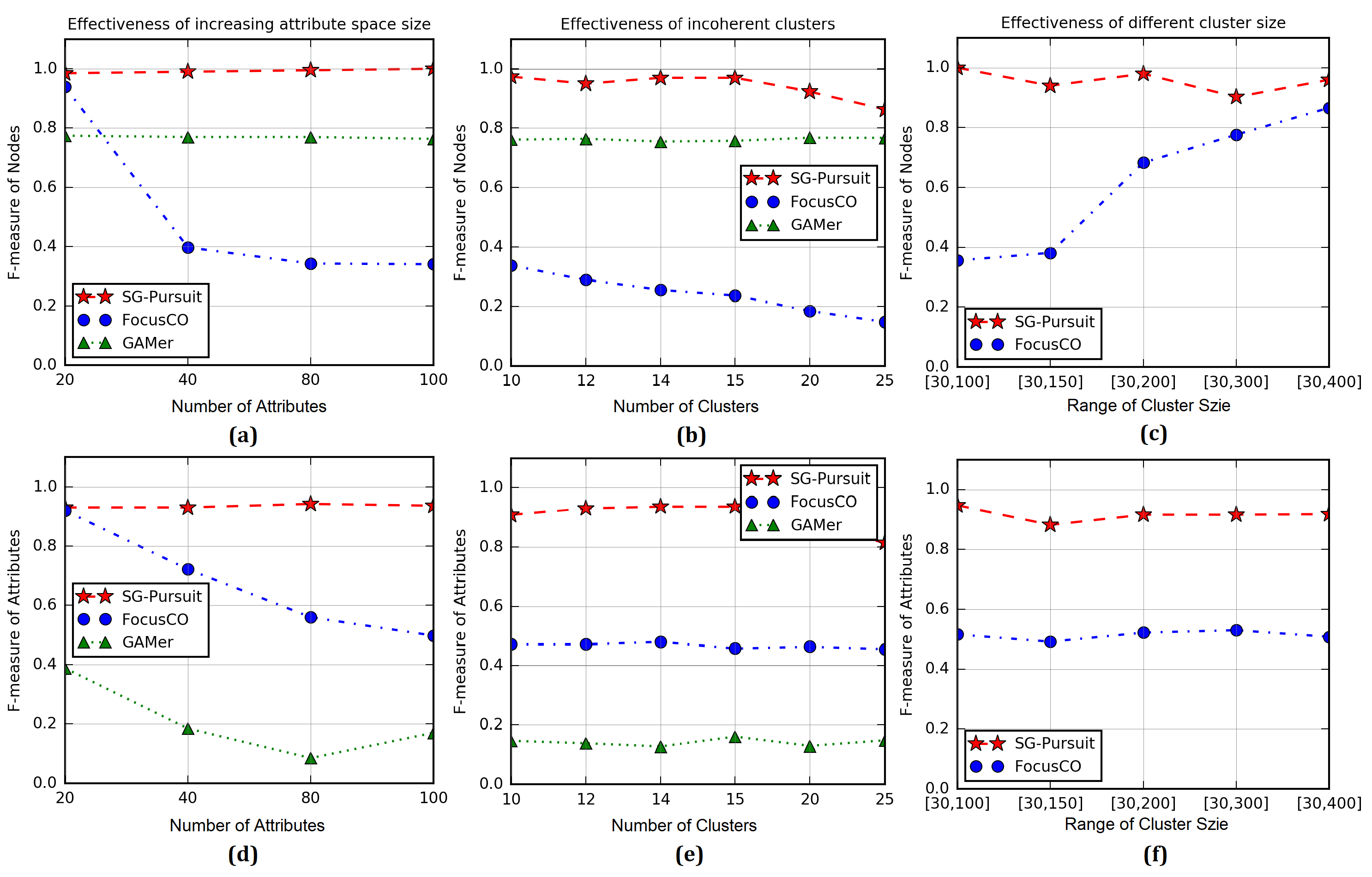}
\vspace{-8mm}
\captionof{figure}{Comparison on {\color{red}F-measures} of detected nodes (first line) and of detected attributes (second line): (a) and (d) for changing the total number of attributes (with 10 coherent attributes); (b) and (e) for increasing number of clusters (with one coherent cluster); and (c) and (f) for changing cluster size variance (graph has variable size clusters), and measurements of \texttt{Gamer} are not shown in these two sub-figures as \texttt{Gamer} is unscalable to detection of clusters of size above 30. } \label{fig:dense-coherent-subgraphs-fscore}
\vspace{-12pt}  
\end{figure*}

\begin{table}[!ht]
\vspace{-3mm}
\caption{Example score (interestingness) functions.}
\vspace{-2mm}
\centering
\resizebox{\columnwidth}{!}{%
\begin{threeparttable}
\begin{tabular}{ | p{2.8cm}|p{6.1cm}|}
\hline
 \textbf{Score Function}\tnote{1} & \textbf{Definition}  \\ 
\hline
 Coherence score& $ x^\intercal \Big(W \odot W\Big) y  - \frac{1}{0.01} x^\intercal \Big(W - 1 \frac{x^\intercal W}{1^\intercal x}\Big) \odot \Big(W - 1 \frac{x^\intercal W}{1^\intercal x}\Big) y  
   - \frac{1}{2} \|x\|_2^2 - \frac{1}{2} \|y\|_2^2$\\  \hline
  Elevated mean scan static &  $x^\intercal Wy / \sqrt{x^\intercal \textbf{1}} - \frac{1}{2} \|x\|_2^2 - \frac{1}{2} \|y\|_2^2 $\\ \hline
 Fisher's test statistic & $x^\intercal Wy - \frac{1}{2} \|x\|_2^2 - \frac{1}{2} \|y\|_2^2 $
 \\ \hline
 Negative square error & $- \|c - W^\intercal x - y\|_2^2 - \frac{1}{2} \|x\|_2^2 - \frac{1}{2} \|y\|_2^2$ \\ \hline
 Logistic function & $\sum_{i=1}^p \Big(y_i\log g(x^\intercal w_i ) + (1-y_i)\log(1 - g(x^\intercal w_i ))\Big) - \frac{1}{2} \|x\|_2^2 - \frac{1}{2} \|y\|_2^2$
  \\ \hline

 \end{tabular}
 \begin{tablenotes}
    \item[1] The L2 regularization component ``$- \frac{1}{2} \|x\|_2^2 - \frac{1}{2} \|y\|_2^2$'' is considered in each score function to enforce the stability of maximizing the score function. $W = [w_1, \cdots, w_n]$, $\odot$ is a Hadamard product operator, and $g(z) = 1/(1 + e^{-z})$. 
  \end{tablenotes}
 \end{threeparttable}}
\label{table:scoreFunctions}
\vspace{-3mm}
\end{table}

\begin{figure}[t]
\center
\includegraphics[width=8.5cm]{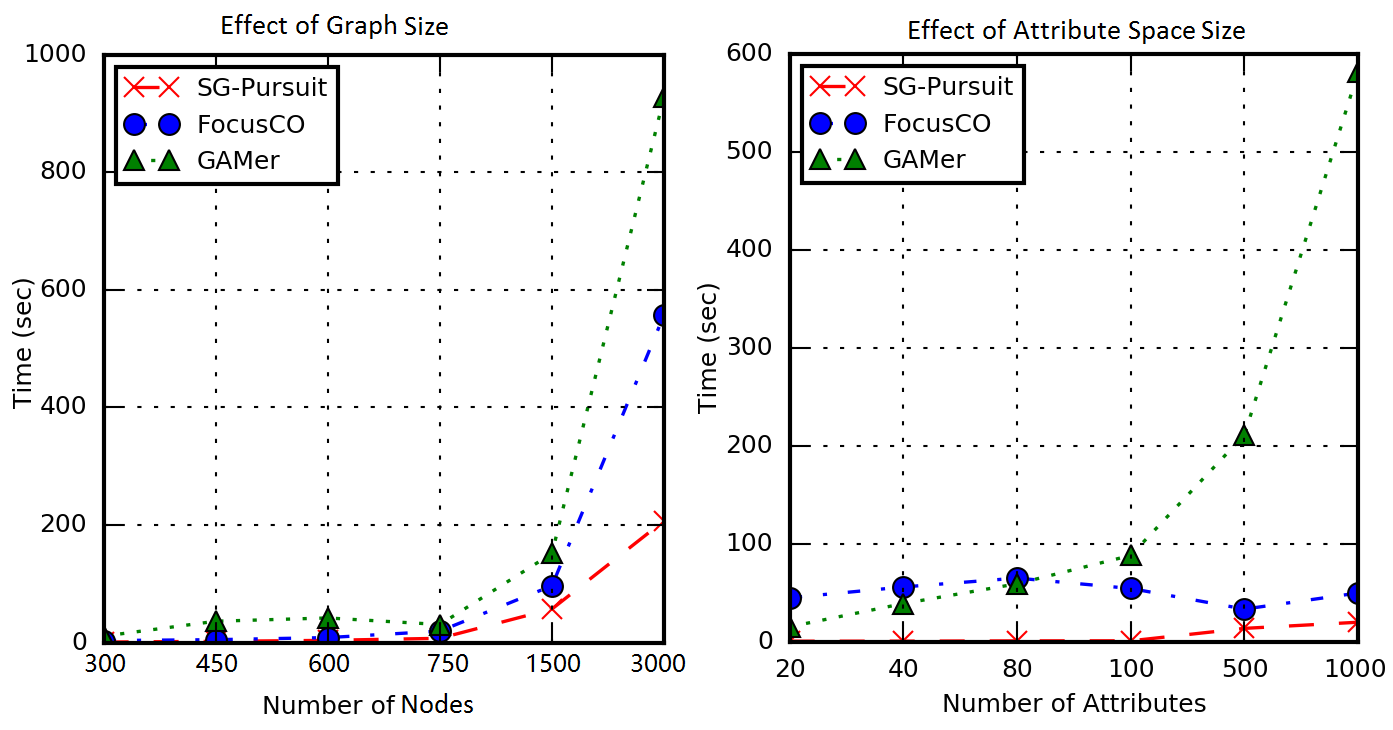}
\vspace{-8mm}
\captionof{figure}{Comparison on running time.} \label{fig:dense-coherent-subgraphs-fscore}
\vspace{-8mm}  
\end{figure}

\section{Experiments}
\label{sect:experiments}
This section thoroughly evaluates the performance of our proposed method on the quality of the detected subspace clusters and run-time on synthetic and real-world networks. The experimental code and data sets are
available from the Link~\cite{fullversion} for reproducibility. 

\subsection{Coherent dense subgraph detection}
\subsubsection{Experimental design} We compared \texttt{SG-Pursuit} with two representative methods, including \texttt{GAMer}~\cite{gunnemann2014gamer} and \texttt{FocusCO}~\cite{perozzi2014focused}.   

\textbf{1. Generation of synthetic graphs:} We used the same generator of synthetic coherent and dense subgraphs as used in the state-of-the-art \texttt{FocusCO} method~\cite{perozzi2014focused}, except that the standard deviation (std) of coherent attributes was set to $\sqrt{0.001}$, instead of $0.001$, which makes the detection problem more challenging. The settings of the other parameters used in \texttt{FocusCO} include: $p_{in} = 0.35$ (density of edges in each cluster) and $p_{out} = 0.1$ (density of edges between clusters). We will compare the performance of different methods based on different combinations of the following parameters: 1) the number of incoherent clusters, 2) the number of coherent attributes, 3) the total number of attributes, and 4) cluster size. We set these parameters to 9, 10, 100, 30, respectively, by default. Note that we set the size of all coherent and incoherent clusters to 30, as \texttt{GAMer} is not scalable to detection of clusters of size larger than 30. We generated one coherent dense cluster and multiple incoherent dense clusters in each synthetic graph. 

\textbf{2. Real-world data.} We used five public benchmark real-world attributed network datasets, including DBLP, Arxiv, Genes, IMDB, and DFB (German soccer premier league data), which are available from and described in details in~\cite{benchmarks}. The basic statistics of these five datasets are provided in Table~\ref{table:coherent-dense-subgraphs-datasets}, with the numbers of nodes ranging from 100 to 11,989; the numbers of edges ranging from 1,106 to 119,258; and the number of attributes ranging from 5 to 300. 

\textbf{3. Implementation and parameter tuning:} The implementations of \texttt{FocusCO} and \texttt{Gamer} are publicly released by authors \footnote{Available at \url{https://github.com/phanein/focused-clustering} and \url{http://dme.rwth-aachen.de/en/gamer}}. 
 \textbf{\texttt{FocusCO}} requires an exemplar set of nodes and has a trade-off parameter $\gamma$ that is used in learning of feature weights. We have tried the representative values of $\gamma$:  $\{0.0,0.0001,0.001,0.01,0.1,0.5,1.0$, $2.0, 3.0, 4.0,5.0,10.0,15.0,
 20.0,25.0,30.0,50.0,100.0\}$ on the graphs, and identified the best value $\gamma = 1$ (with the largest overall F-measure), which is also the default value used in \texttt{FocusCO}. In order to make \texttt{FocusCO} the best competitive to our method, we used a random set of $90\%$ nodes in each coherent dense subspace cluster as the input exemplar set of nodes. \texttt{FocusCO} estimates a weight for each attribute that characterizes the importance of this attribute, and return the top $s$ attributes with the largest weights as the set of coherent attributes, and set $s$ to the true number of coherent attributes. \textbf{\texttt{GAMer}} has four main parameters, including $s_{min}$ (the minimum number of coherent attributes), $\gamma_{min}$ (the minimum threshold on density), and $n_{min}$ (the minimum cluster size), and $w$ (the maximum width that control the level of coherence). We followed the recommended strategies by the authors and identified the best parameter values for \texttt{FocusCO} and \texttt{Gamer}. In particular, these parameters for \textbf{synthetic data sets} were set as follows:
\begin{itemize}
\item $n_{min}$ (minimum cluster size): $\{0.5 s, s\}$, where $s$ is the size of the true coherent cluster.
\item $s_{min}$ (the minimum number of coherent attributes):  $\{0.5 k, k\}$, where $k$ is the number of the true coherent attributes. 
\item $\gamma_{min}$ (the minimum threshold on density): 0.35 (the density of the true coherent cluster). 
\item $w$ (the maximum width that controls the level of coherence): 0.1, which is around 3 times $\sigma$, where $\sigma = \sqrt{0.001}$ is the standard deviation of coherent attributes. 
\end{itemize}
For the five real-world data sets, as the ground truth labels are unavailable, we followed the recommended strategies by the authors and identified the best parameter values for \textbf{\texttt{GAMer}}~\cite{gunnemann2010subspace, gunnemann2014gamer}. We tried different combinations of the four major parameters and returned the best results:  $n_{min}=\{2,3,4,5,10,15,20\}$, $s_{min}=\{2,3,4,5,10,15,20,25,30\}$, $\gamma=\{0.3,0.4,0.5,0.6,0.7,0.8,0.9\}$, and $w=\{0.01,0.1,1.0,5.0,10.0,50.0,100.0\}$.

We did not consider other related papers that focus on different objectives rather than only cluster density (e.g., the normalized subspace graph cut objective as considered in~\cite{gunnemann2013spectral}) and also their implementations are not publicly available. 

\textbf{4. Settings of our proposed method \texttt{SG-Pursuit}}. We used the following score function to detect the most coherent dense subspace cluster in each synthetic graph: $f(x, y) = x^T \Big(W \odot W\Big) y  - \frac{1}{0.01} x^T \Big(W - 1 \frac{x^T W}{1^T x}\Big) \odot \Big(W - 1 \frac{x^T W}{1^T x}\Big) y  
   - \frac{1}{2} \|x\|_2^2 - \frac{1}{2} \|y\|_2^2 + \lambda \frac{x^T A x}{1^T x}$, where $A$ is the adjacency matrix of the input graph, and $\lambda$ is a tradeoff parameter to balance to coherence score (See Table~\ref{table:scoreFunctions}) and the density score $\frac{x^T A x}{1^T x}$. The parameter $\lambda$ was set to 5. We applied projected gradient descent to solve the subproblem in Line 10 of \texttt{SG-Pursuit}. The parameters $k$ (upper bound of the cluster size) and $s$ (upper bound on coherent attributes) were set to the true cluster size and number of coherent attributes, respectively, in the synthetic datasets. For the real-world datasets, we tested the ranges $k\in \{5, 10, 15, 30\}$ and $s\in \{3,5,10,15,20\}$ and identified the settings with the best objective scores. 

\textbf{5. Evaluation metrics}. Each \textbf{synthetic graph} has a single true coherent dense subspace cluster (a combination of a subset of nodes and a subset of attributes) and the task was to detect this cluster. We reported the F-measures of the subsets of nodes and attributes for each competitive method. We note that \texttt{FocusCO} and \texttt{GAMer} may return multiple candidate clusters in an input graph, and in this case we return the cluster with the highest F-measure in order to make fair comparisons. We generated 50 synthetic graphs for each setting and reported the average F-measure and running time. For the five \textbf{real-world attributed network datasets}, where no ground truth is
given, we considered three major measures, including \textbf{average cluster density}, \textbf{average cluster size}, and \textbf{average coherence distance}. The average cluster density is defined as the average degree of nodes within the $K$ subspace clusters identified, where $K$ is predefined. The coherence distance of a specific subspace cluster is defined as the average Euclidean distance between the nodes in this cluster based on the subset of attributed selected. The average coherence distance is the average of the coherence distances of the $K$ subspace clusters. A combination of \underline{\smash{a high average cluster density}}, \underline{\smash{a high average cluster size}}, and \underline{\smash{a low average coherence distance}} indicates a high overall quality of the clusters detected.

\subsubsection{Quality Analysis} 1) \textbf{Synthetic data with ground truth labels}. 
The comparison on F-measures among the three competitive methods is shown in Figure~\ref{fig:dense-coherent-subgraphs-fscore}, by varying total number of irrelevant attributes, number of incoherent clusters, and cluster size variance. The results indicate that \texttt{SG-Pursuit} significantly outperformed \texttt{FocusCO} and \texttt{GAMer} with more than \underline{\smash{15 percent}} marginal improvements in overall on F-measures of the detected nodes and the detect coherent attributes. As shown in Figure 4(c), when the cluster size increases, the F-measure of \texttt{FocusCO} consistently increases. In particular, we observed that when the cluster size is above 150, \texttt{FocusCO} achieved F-measure close to 1.0. In addition, when the standard deviation of coherent attributes decreases (in the shown Figures, we fixed this to $\sqrt{0.001}$), \texttt{FocusCO} performed better for large cluster sizes. To summarize, \texttt{SG-Pursuit} was more robust to \texttt{FocusCO} on low levels of coherence and small cluster sizes.   2) \textbf{Real-world data}. As the real-world datasets do not have ground truth labels, we can not apply \texttt{FocusCO} since it requires a predefined subset of ground truth nodes. Hence, we focus on the comparison between \texttt{SG-Pursuit} and \texttt{GAMer} with different predefined numbers of clusters ($\text{Top-}K, K=5,10,15,20$) . As shown in Table~\ref{table:coherent-dense-subgraphs-datasets}, \texttt{SG-Pursuit} was able to identify subspace clusters with the three major measures coherently better than those of the clusters returned by \texttt{GAMer} in most of the settings. \texttt{GAMer} was able to identify clusters with densities larger than those detected by \texttt{SG-Pursuit}, but with much smaller cluster sizes and much large coherence distances. 

\subsubsection{Scalability analysis} 
The comparison on running times of competitive methods is shown in Figure~\ref{fig:dense-coherent-subgraphs-fscore} with respect to varying numbers of attributes and nodes. The results indicate that \texttt{SG-Pursuit} was faster than both \texttt{FocusCO} and \texttt{GAMer} over several orders of magnitude. The running time of \texttt{FocusCO} was independent on the number of attributes, but increases quadratically on the number of nodes (graph size). The running time \texttt{GAMer} increases quadratically on both numbers of attributes and nodes.

\begin{table*}[!ht]
\vspace{-3mm}
\caption{Analysis of five real-world datasets for coherent dense subspace cluster (subgraph) detection.}
\vspace{-2mm}
\centering
\begin{tabular}{|c|c|c|c|c|c|c|c|c|c|c| }\hline
\multirow{2}{*}{Dataset}& \multirow{2}{*}{Node}& \multirow{2}{*}{Edge}& \multirow{2}{*}{Attribute}& \multirow{2}{*}{Top-K} &  \multicolumn{2}{c|}{Avg. Cluster density}& \multicolumn{2}{c|}{ Avg. Cluster Size} &  \multicolumn{2}{c|}{Avg. Coherence Distance} \\ \cline{6-11}
& & & & &SG-Pursuit&GAMer&SG-Pursuit&GAMer&SG-Pursuit&GAMer\\ \hline
\multirow{4}{*}{\textbf{DFB}}&\multirow{4}{*}{100}&\multirow{4}{*}{1106} &\multirow{4}{*}{5}&5 & \textbf{9.69} & 5.4 & \textbf{10.8} & 6.4 & \textbf{0.13} & 6.38 \\ \cline{5-11}
& & & & 10 & \textbf{9.94} & 3.9 & \textbf{11.0} & 4.9 & \textbf{0.13} & 7.41\\ \cline{5-11}
& & & & 15 & \textbf{11.03} & 3.33 & \textbf{12.07} & 4.33 & \textbf{0.12} & 7.51\\ \cline{5-11}
& & & & 20 & \textbf{10.55} & 2.9 & \textbf{11.6} & 3.9 & \textbf{0.12} & 8.25\\ 
\hline
\multirow{4}{*}{\textbf{DBLP}} &\multirow{4}{*}{774}&\multirow{4}{*}{1757} &\multirow{4}{*}{20}&5 & \textbf{4.12} & 3.2 & \textbf{5.2} & 4.2 & \textbf{0.0} & \textbf{0.0}\\ \cline{5-11}
& & & & 10 & \textbf{3.84} & 3.2 & \textbf{5.7} & 4.2 & 0.06 & \textbf{0.0}\\ \cline{5-11}
& & & & 15 & \textbf{3.25} & 3.17 & \textbf{7.3} & 4.17 & 0.3 & \textbf{0.0}\\ \cline{5-11}
& & & & 20 & \textbf{3.34} & 3.17 & \textbf{7.68} & 4.17 & 0.4 & \textbf{0.0}\\
\hline
\multirow{4}{*}{\textbf{IMDB}}&\multirow{4}{*}{862}&\multirow{4}{*}{4388} &\multirow{4}{*}{21}&5 & 3.24 & \textbf{4.96} & 4.6 & \textbf{10} & \textbf{0.0} & \textbf{0.0}\\ \cline{5-11}
& & & & 10 & 3.17 & \textbf{4.52} & 5.67 & \textbf{10.0} & 0.08 & \textbf{0.0}\\ \cline{5-11}
& & & & 15 & \textbf{3.52} & 2.97 & \textbf{5.79} & 4.27 & 0.06 & \textbf{0.0}\\ \cline{5-11}
& & & & 20 & \textbf{3.36} & 2.78 & \textbf{5.37} & 4.2 & 0.06 & \textbf{0.0}\\ 
\hline
\multirow{4}{*}{\textbf{Genes}}&\multirow{4}{*}{2900}&\multirow{4}{*}{8264} &\multirow{4}{*}{115}&5 & 3.76 & \textbf{6.92} & \textbf{24.8} & 10.6 & \textbf{0.29} & 8.28\\ \cline{5-11}
& & & & 10 & 3.74 & \textbf{5.91} & \textbf{24.7} & 10.8 & \textbf{0.31} & 5.46\\ \cline{5-11}
& & & & 15 & 3.82 & \textbf{5.58} & \textbf{23.09} & 10.87 & \textbf{0.33} & 4.69\\ \cline{5-11}
& & & & 20 & 3.85 & \textbf{5.36} & \textbf{21.83} & 10.9 & \textbf{0.33} & 4.2\\

\hline
\multirow{4}{*}{\textbf{Arxiv}}&\multirow{4}{*}{11989}&\multirow{4}{*}{119258} &\multirow{4}{*}{300}&5 & \textbf{11.53} & 4.16 & \textbf{15.8} & 10.0 & \textbf{0.0} & \textbf{0.0}\\ \cline{5-11}
& & & & 10 & \textbf{9.65} & 4.24 & \textbf{12.4} & 10.0 & \textbf{0.0} & \textbf{0.0}\\ \cline{5-11}
& & & & 15 & \textbf{9.03} & 4.23 & \textbf{11.27} & 10.0 & \textbf{0.0} & \textbf{0.0}\\ \cline{5-11}
& & & & 20 & \textbf{8.72} & 4.21 & \textbf{10.7} & 10.0 & \textbf{0.0} & \textbf{0.0}\\
\hline
\end{tabular}
\vspace{-2mm}\label{table:coherent-dense-subgraphs-datasets}
\end{table*}

\subsection{Anomalous connected cluster detection}

\subsubsection{Experimental design} We considered two representative methods, including \texttt{AMEN}~\cite{perozzi2016scalable} and \texttt{SODA}~\cite{gupta2014local}.

\textbf{1. Data sets:} 1) \textbf{Chicago Crime Data}. A data set of crime data records in Chicago was collected form the official website~``\url{https://data.cityofchicago.org/}'' from 2010 to 2014 that has 1,515,241 crime records in total, each of which has the location information (latitude and longitude), crime category (e.g., BATTERY, BURGLARY, THEFT), and description (e.g., ``aggravated domestic battery: knife / cutting inst''). There are 35 different crime categories in total. We collected the census-tract-level graph in Chicago from the same website that has 46,357 nodes (census tracts) and 168,020 edges in total, and considered the frequency of each keyword in the descriptions of crime records as an attribute. There are 121 keywords in total that are non-stop-words and have frequencies above 10,000, which are considered as attributes. 
In order to generate a ground-truth anomalous connected cluster of nodes, we picked a particular crime type (BATTERY or BURGLARY), identified a connected subgraph of size 100 via random walk, and then removed the crime records of this particular category in all nodes outside this subgraph, which generated a rare category as an anomalous category.  This subgraph was considered as an anomalous cluster for crime records of categories that are different from this specific category, and the keywords that are specifically relevant to this category were considered as ground-truth anomalous attributes. We tried this process 50 times to generate 50 anomalous connected clusters, and manually identified 22 keywords relevant to BATTERY and 5 keywords relevant to BURGLARY as anomalous attributes. 
2) \textbf{Yelp Data}. A Yelp reviews data set was publicly released by Yelp for academic research purposes\footnote{Available at \url{http://www.yelp.com/dataset_challenge}}. All restaurants and reviews in the U.S. from 2014 to 2015 were considered, which includes 25,881 restaurants and 686,703 reviews. The frequencies of 1,151 keywords in the reviews that are non-stop-words and have frequencies above 5,000 are considered as attributes. We generated a geographic network of restaurants (nodes), in which each restaurant is connected to its 10 nearest restaurants, and there are 244,012 edges in total. We used the sample strategy as in the Chicago Crime Data to generate 50 ground-truth anomalous connected clusters of size 100 for the specific category ``Mexican''. 
\textbf{2. Implementation and parameter tuning:} The implement ions of \texttt{AMEN} and \texttt{SODA} are publicly released by the authors\footnote{Available at \url{https://github.com/phanein/amen/tree/master/amen} and \url{https://github.com/manavs19/subgraph-outlier-detection}}. Their  parameters were tuned by the recommended strategies by the authors. In particular, both methods require the definition of canidate neighborhoods for scanning. A neighborhood is defined a subset that includes a focus node and the nodes who are $k$-step nearest neighbors to the focus node. If $k=1$, a neighborhood is also called an ego network. We considered the possible values $k\in \{1, 2, 3, 4, 5\}$. Therefore, there were $5n$ candidate neighborhoods in total for each of the two methods with the sizes ranging around $10$ to $300$ nodes. For our proposed method \texttt{SG-Pursuit}, we considered the elevated mean statistic function as defined in Table~\ref{table:scoreFunctions}. The upper bound of cluster size $k$ was set to 100. The upper bound of number of attributes $s$ was set to 22 for BATTERY related anomalous clusters and $5$ for BURGLARY related anomalous clusters. 

\begin{table}[!htb]
\caption{Chicago Crime data (Fm refers to F-Measure).}
\vspace{-2mm}
\centering
\resizebox{\columnwidth}{!}{%
\begin{tabular}{|c||c|c|c|c|c|c|c|c| }\hline
Methods&type& Node Fm &  Attribute Fm& Running Time (s) \\ \hline 
\multirow{2}{*}{\textbf{SODA}} & BATTERY   & 0.476&  0.146& 7,997.893\\
&BURGLARY& 0.45&  0.020& 11,043.892\\\hline
\multirow{2}{*}{\textbf{AMEN}} & BATTERY &	0.363& 0.818& 3,835.589\\
&BURGLARY&	0.337& 0.800& 4,265.449\\\hline
\multirow{2}{*}{\textbf{SG-Pursuit}} &BATTERY& \textbf{0.683}& \textbf{0.955} & \textbf{73.998}\\
&BURGLARY&	\textbf{0.538}& \textbf{1.000}& \textbf{37.538}\\
\hline
\end{tabular}}
\label{table:crimedata}
\end{table}

\subsubsection{Quality and Scalability Analysis}
The detection results of the competitive methods on the Chicago Crime Data are shown in Table~\ref{table:crimedata}. The results indicate that \texttt{SG-Pursuit} outperformed \texttt{SODA} and \texttt{AMEN} on F-Measure of nodes with more than 20\% marginal improvements, and on F-measure of attributes with around 15\% marginal improvements. The running time of \texttt{SG-Pursuit} was less than those of \texttt{SODA} and \texttt{AMEN} on several orders of magnitude. The results of our method on Yelp Data contain three parts: 1) The quality of returned clusters: The F-measure of the returned clusters is 0.31 with the precision 0.314 and the recall 0.309;  2) The top 10 most frequent keyword pairs, i.e. (frequency, keyword), returned are (21, ``tacos''), (21, ``asada'')
(20, ``taco''), (19, ``salsa'')
, (19, ``level''), (15, ``vegas'')
(14, ``mexican''), (14, ``item'')
(14, ``beans''), and (13, ``worth''), where the frequency of a keyword refers to the number of times that this keyword occurs in the anomalous subspace clusters detected by \texttt{SG-Pursuit}. 6 out of 10 keywords are related with "Mexican", which demonstrates that our method can identify the related keywords on the specified category; 3) The running time of our algorithm was 6.98 minutes. 
We were not able to obtain results from \texttt{AMEN} and \texttt{SODA} after running several hours. These baseline methods cannot handle graphs that have more than 10,000 nodes and 1,000 attributes.

\section{Conclusions}
\label{sect:conclusions}
This paper presents \texttt{SG-Pursuit}, a novel generic algorithm to subspace cluster detection in multi-attributed networks that runs in nearly-linear time and provides rigorous guarantees, including a geometrical  convergence rate and a tight error bound. Extensive experiments demonstrate the effectiveness and efficiency of our algorithms. For the future work, we plan to generate our algorithm to subspace cluster detection in heterogeneous networks. \vspace{-2mm}

\appendix

\section{Proof of Theorem~2.2}
\label{proof-theorem-2-2}
\begin{proof}
Let $r_x^{i+1} = x^{i+1} - x^*$ and  $r_y^{i+1} = y^{i+1} - y^*$. Then the component $\|r_x^{i+1}\|_2 + \|r_y^{i+1}\|_2$ is upper bounded as 
\begin{eqnarray}
\|r_x^{i+1}\|_2 + \|r_y^{i+1}\|_2 = \|x^{i+1} - x^*\|_2 + \|y^{i+1} - y^*\|_2  =\nonumber \\
\|x^{i+1} - b_x + b_x - x^*\|_2 + \|y^{i+1} - b_y + b_y - y^*\|_2  \le\nonumber \\
\|b_x - x^{i+1}\|_2 + \|b_x - x^*\|_2 + \|b_y - y^{i+1}\|_2 + \| b_y - y^*\|_2  =\nonumber \\
\|b_x - [b_x]_{\Psi_x^{i+1}} \|_2 + \|b_y - [b_y]_{\Psi_y^{i+1}} \|_2 + \|b_x - x^*\|_2 + \| b_y - y^*\|_2 \le \nonumber  \\
 c_T\|b_x - x^* \|_2 + \|b_y - y^* \|_2 + \|b_x - x^*\|_2 + \| b_y - y^*\|_2  = \nonumber \\
(c_T + 1) \|b_x - x^* \|_2 + 2 \| b_y - y^*\|_2  
\le \nonumber \\
(c_T + 1) \Big(\|b_x - x^* \|_2 + \| b_y - y^*\|_2\Big), \nonumber 
\end{eqnarray}
where the first inequality follows from the use of the triangle inequality; the second equality follows from the fact that $x^{i+1} = [b_x]_{\Psi_x^{i+1}}$ and $y^{i+1} = [b_y]_{\Psi_y^{i+1}}$; the second inequality follows from  the definition of the tail projection oracle $\text{T}(\cdot)$ and the fact that $\Psi_x^{i+1} = \text{T}(b_x^i)$,  $\Psi_y^{i+1} = \arg\max_{R \subseteq \{1, \cdots, p\}}\{\|[b_y^i]_R\|_2^2: \|R\|_0 \le  s\}$, $\text{supp}(x^*) \in \mathbb{M}(k)$, and $\|y^*\|_0 \le s$; and the last inequality follows from the fact that $c_T \ge 1$. Recall that $[b_y]_{\Psi_y^{i+1}}$ refers to the projected vector in the subspace defined by the subset $\Psi_y^{i+1}$. Denote by $w$ the projected vector $[b_y]_{\Psi_y^{i+1}}$. $w$ is defined as: $w_i = [b_y]_i$, the $i$-th entry in the vector $b_y$, if $i \in \Psi_y^{i+1}$; otherwise, $w_i = 0$.

Recall that $\Omega_x= \Gamma_x\cup \text{supp}(x^i)$ and $\Omega_y= \Gamma_y\cup \text{supp}(y^i)$. The component $\|(x^* - b_x)_{\Omega_x}\|_2 + \|(y^* - b_y)_{\Omega_y}\|_2$ is upper bounded as
\begin{eqnarray}
\|(x^* - b_x)_{\Omega_x}\|_2 + \|(y^* - b_y)_{\Omega_y}\|_2 = \nonumber \\
\left\langle b_x - x^*, \frac{(b_x - x^*)_{\Omega_x}}{\|(x^* - b_x)_{\Omega_x}\|_2} \right\rangle + \left\langle b_y - y^*, \frac{(b_y - y^*)_{\Omega_y}}{ \|(y^* - b_y)_{\Omega_y}\|_2} \right\rangle =  \nonumber \\
\Big\langle b_x - x^* - \frac{\gamma^-}{(\gamma^+)^2} [\nabla_x f(b_x, b_y)]_{\Omega_x} + \frac{\gamma^-}{(\gamma^+)^2} [\nabla_x f(x^*, y^*)]_{\Omega_x}, \nonumber \\ \frac{(b_x - x^*)_{\Omega_x}}{\|(x^* - b_x)_{\Omega_x}\|_2} \Big\rangle +  \nonumber
\end{eqnarray}
\begin{eqnarray}
\Big\langle b_y - y^*  - \frac{\gamma^-}{(\gamma^+)^2} [\nabla_y f(b_x, b_y)]_{\Omega_y} + \frac{\gamma^-}{(\gamma^+)^2} [\nabla_y f(x^*, y^*)]_{\Omega_y}, \nonumber \\ \frac{(b_y - y^*)_{\Omega_y}}{ \|(y^* - b_y)_{\Omega_y}\|_2} \Big\rangle  - 
\nonumber\\
\left\langle\frac{\gamma^-}{(\gamma^+)^2} [\nabla_x f(x^*, y^*)]_{\Omega_x},\frac{(b_x - x^*)_{\Omega_x}}{\|(x^* - b_x)_{\Omega_x}\|_2}  \right\rangle - \nonumber \\
\left\langle\frac{\gamma^-}{(\gamma^+)^2} [\nabla_y f(x^*, y^*)]_{\Omega_y}, \frac{(b_y - y^*)_{\Omega_y}}{ \|(y^* - b_y)_{\Omega_y}\|_2} \right\rangle \le  \nonumber \\
\sqrt{2}\rho \Big(\|b_x - x^*\|_2 
+ \|b_y - y^*\|_2\Big) + \nonumber \\ \frac{\gamma^-}{(\gamma^+)^2} \Big(\left\|[\nabla_x f(x^*, y^*)]_{\Omega_x}\right\|_2 + \left\|[\nabla_y f(x^*, y^*)]_{\Omega_y}\right\|_2\Big) \le \nonumber \\
\sqrt{2}\rho \Big(\|b_x - x^*\|_2 + \|b_y - y^*\|_2\Big) + \frac{\gamma^-}{(\gamma^+)^2}\Big(\varepsilon_x + \varepsilon_y\Big),
\nonumber
\end{eqnarray}
where the second equality follows from the fact that $(b_x, b_y)$ is the optimal solution to the sub-problem in Line 10 of Algorithm~\ref{alg:SG-Pursuit} and hence $[\nabla_x f(b_x, b_y)]_{\Omega_x} = 0$ and $[\nabla_y f(b_x, b_y)]_{\Omega_y} = 0$; the first inequality follows from (1) the fact that $\langle w, v\rangle \le \|w\|_2 \|v\|_2$ for any vectors $w$ and $v$ and (2) the inequality~(\ref{RSC-RSS-ext4}) in Lemma~\ref{rss-rsc-properties} by letting $\xi = \frac{\gamma^-}{(\gamma^+)^2}$, given that $\text{supp}(b_x), \text{supp}(x^*) \in \mathbb{M}(2k)$, $\Omega_x \subseteq \mathbb{M}(4k)$, and $|\Omega_y| \le 4s$; and the last inequality follows from the definitions of $\varepsilon_x$ and $\varepsilon_y$ in Theorem~\ref{theorem:accuracy} and the fact that $\Omega_x \in \mathbb{M}(2k)$ and $\|\Omega_y\|_0 \le 3s$: 
\[\varepsilon_x = \max_{S \in \mathbb{M}(2k)} \|[\nabla f_x(x^*, y^*)]_S\|_2^2,\] and \[\varepsilon_y = \max_{R\subseteq\{1, \cdots, p\}, |R| \le 3s } \|[\nabla f_y(x^*, y^*)]_R\|_2^2.\]
It follows that 
\begin{eqnarray}
\|x^* - b_x\|_2 + \|y^* - b_y\|_2 \le 
\|(x^* - b_x)_{\Omega_x}\|_2 + \|(y^* - b_y)_{\Omega_y}\|_2 + \nonumber \\
\|(x^* - b_x)_{\Omega_x^c}\|_2 + \|(y^* - b_y)_{\Omega_y^c}\|_2 \le \nonumber \\
\sqrt{2}\rho \Big(\|b_x - x^*\|_2 + \|b_y - y^*\|_2\Big) + \frac{\gamma^-}{(\gamma^+)^2}(\varepsilon_x + \varepsilon_y) + \nonumber \\
\|(x^* - b_x)_{\Omega_x^c}\|_2 + \|(y^* - b_y)_{\Omega_y^c}\|_2,
\nonumber 
\end{eqnarray}
where the first inequality follows from the use of triangle inequality, and the second inequality follows from the inequality obtained above. After rearrangement, we obtain 
\begin{eqnarray}
\|x^* - b_x\|_2 + \|y^* - b_y\|_2  \nonumber \\
\le\frac{1 }{1 -  \sqrt{2}\rho}\Big( \|(x^* - b_x)_{\Omega_x^c}\|_2 + \|(y^* - b_y)_{\Omega_y^c}\|_2 
+ \frac{\gamma^-}{(\gamma^+)^2}(\varepsilon_x + \varepsilon_y)\Big) \nonumber \\
= \frac{1 }{1 -  \sqrt{2}\rho}\Big( \|x_{\Omega_x^c}^*\|_2 + \|y_{\Omega_y^c}^*\|_2 +  \frac{\gamma^-}{(\gamma^+)^2}(\varepsilon_x + \varepsilon_y)\Big)\nonumber \\
= \frac{1}{1 -  \sqrt{2}\rho}\Big( \|(x^*-x^i)_{\Omega_x^c}\|_2 + \|(y^*-y^i)_{\Omega_y^c}\|_2 
+ \frac{\gamma^-}{(\gamma^+)^2}(\varepsilon_x + \varepsilon_y)\Big)\nonumber \\
=\frac{1}{1 -  \sqrt{2}\rho}\Big( \|[r^i_x]_{\Omega_x^c}\|_2 + \|[r^i_y]_{\Omega_y^c}\|_2
 + \frac{\gamma^-}{(\gamma^+)^2}(\varepsilon_x + \varepsilon_y)\Big)\nonumber \\
\le  \frac{1}{1 -  \sqrt{2}\rho}\Big( \|[r^i_x]_{\Gamma_x^c}\|_2 + \|[r^i_y]_{\Gamma_y^c}\|_2 
 + \frac{\gamma^-}{(\gamma^+)^2}(\varepsilon_x + \varepsilon_y)\Big),
\nonumber 
\end{eqnarray}
where the first equality follows from the fact that $\text{supp}(b_x) \in \Omega_x$ and $\text{supp}(b_y) \in \Omega_y$ and hence $[b_x]_{\Omega_x^c} = 0$ and $[b_y]_{\Omega_y^c} = 0$; the second equality follows from the fact that $\text{supp}(x^i) \subseteq \Omega_x$ and $\text{supp}(y^i) \subseteq \Omega_y$, and hence $[x^i]_{\Omega_x^c} = 0$ and $[y^i]_{\Omega_y^c} = 0$; and the last inequality follows from the fact that $\Gamma_x \subseteq \Omega_x$ and $\Gamma_y \subseteq \Omega_y$, and hence $\Omega_x^c \subseteq \Gamma_x^c$ and $\Omega_y^c \subseteq \Gamma_y^c$. 

Combining the above inequalities, we obtain 
\begin{eqnarray}
\|r_x^{i+1}\|_2 + \|r_y^{i+1}\|_2 \le \frac{(c_T + 1) }{1 -  \sqrt{2}\rho}\Big( \|[r^i_x]_{\Gamma_x^c}\|_2 + \|[r^i_y]_{\Gamma_y^c}\|_2 +  \nonumber \\
\frac{\gamma^-}{(\gamma^+)^2}(\varepsilon_x + \varepsilon_y)\Big). \nonumber 
\end{eqnarray}
From Lemma~\ref{lemma-bound}, we have 
\begin{eqnarray}
\|[[{r_x^i}]_{\Gamma_x^c},  [{r_y^i}]_{\Gamma_y^c}]\|_2    
\le \sqrt{1 - \alpha_0^2}   \|[{r_x^i}, {r_y^i}]\|_2 + \Big(\frac{\alpha_0 \beta_0  }{1 - \alpha_0^2} + \frac{\beta_0}{\alpha_0}\Big) (\varepsilon_x + \varepsilon_y), \nonumber
\end{eqnarray}
where $\alpha_0 =c_H(1-\rho) -  \rho$, $\rho = \sqrt{1 - \Big(\frac{\gamma^-}{\gamma^+}\Big)^2}$, and $\beta_0 = (c_H + 1) \frac{\gamma^-}{(\gamma^+)^2}$.
Given that 
\[
\|[r^i_x]_{\Gamma_x^c}\|_2 + \|[r^i_y]_{\Gamma_y^c}\|_2 \le \sqrt{2} \|[[{r_x^i}]_{\Gamma_x^c},  [{r_y^i}]_{\Gamma_y^c}]\|_2,
\]
we have 
\begin{eqnarray}
\|[r^i_x]_{\Gamma_x^c}\|_2 + \|[r^i_y]_{\Gamma_y^c}\|_2    
\le \sqrt{2 - 2\alpha_0^2}   \Big(\|{r_x^i}\|_2 + \|{r_y^i}]\|_2\Big) + \nonumber \\\Big(\frac{\sqrt{2}\alpha_0 \beta_0  }{1 - \alpha_0^2} + \frac{\sqrt{2}\beta_0}{\alpha_0}\Big) (\varepsilon_x + \varepsilon_y). \nonumber
\end{eqnarray}
Combining the above inequalities, we obtain 
\begin{eqnarray}
\|r_x^{i+1}\|_2 + \|r_y^{i+1}\|_2 \le \alpha \Big(\|{r_x^i}\|_2 + \|{r_y^i}]\|_2\Big) +  \beta (\varepsilon_x + \varepsilon_y).\nonumber
\end{eqnarray}
where $\alpha = \frac{(c_T + 1) \sqrt{2 - 2\alpha_0^2} }{1 -  \sqrt{2}\rho}$ and $\beta = \frac{(c_T + 1) }{1 -  \sqrt{2}\rho}\Big(\frac{\gamma^-}{(\gamma^+)^2} +\Big(\frac{\sqrt{2}\alpha_0 \beta_0  }{1 - \alpha_0^2} + \frac{\sqrt{2}\beta_0}{\alpha_0}\Big)  \Big)$. 
\end{proof}

 \begin{lemma} Let $r^i_x = x^i - x^*$, $r^i_y = y^i - y^*$, $\Gamma_x = \text{H}(\nabla_x f(x^i, y^i))$, and $\Gamma_y = \arg\max_{R \subseteq \{1, \cdots, p\}}\{\|[\nabla_y f(x^i, y^i)]_R\|_2^2: |R| \le 2 s\}$. Then 
\begin{eqnarray}
\|[[{r_x^i}]_{\Gamma_x^c},  [{r_y^i}]_{\Gamma_y^c}]\|_2    
\le \sqrt{1 - \alpha_0^2}   \|[{r_x^i}, {r_y^i}]\|_2 + \Big(\frac{\alpha_0 \beta_0  }{1 - \alpha_0^2} + \frac{\beta_0}{\alpha_0}\Big) (\varepsilon_x + \varepsilon_y) \nonumber
\end{eqnarray}
where $\alpha_0 =c_\text{H}(1-\rho) -  \rho$, $\rho = \sqrt{1 - \Big(\frac{\gamma^-}{\gamma^+}\Big)^2}$, and $\beta_0 = (c_H + 1) \frac{\gamma^-}{(\gamma^+)^2}$.\label{lemma-bound}
 \end{lemma}

\begin{proof}
Denote $\Phi_{x} = \text{supp}(x^*) \in \mathbb{M}(k)$, $\Phi_{y} = \text{supp}(y^*) \in \{y\ |\ \|y\|_0 \le s \}$, $r^i_x = x^i - x^*$, $\Theta_x = \text{supp}(r^i_x) \in \mathbb{M}(2k)$, $r^i_y = y^i - y^*$, and $\Theta_y = \text{supp}(r^i_y) \in \{y\ |\ \|y\|_0 \le 2s \}$. 
Denote  \[\Gamma_x = \text{H}(\nabla_x f(x^i, y^i)) \in \mathbb{M}(k),\] \[\Gamma_y = \arg\max_{R \subseteq \{1, \cdots, p\}}\{\left\|[\nabla_y f(x^i, y^i)]_R\right\|_2^2: |R| \le 2 s\},\] and 
\begin{eqnarray}\left\|[\nabla_x f(x^i, y^i)]_{\Gamma_x},[\nabla_y f(x^i, y^i)]_{\Gamma_y}\right\|_2= \nonumber\\ \left(\Big\|[\nabla_x f(x^i, y^i)]_{\Gamma_x}\right\|_2^2 +  \left\|[\nabla_y f(x^i, y^i)]_{\Gamma_y}\right\|_2^2\Big)^{1/2}.\nonumber 
\end{eqnarray} From the definitions of $\Gamma_x$ and $\Gamma_y$, we have
\[\left\|[\nabla_x f(x^i, y^i)]_{\Gamma_x}\right\|_2 \ge c_H \left\|[\nabla_x f(x^i, y^i)]_{\Phi_x}\right\|_2\] and \[\left\|[\nabla_y f(x^i, y^i)]_{\Gamma_y}\right\|_2 \ge  \left\|[\nabla_y f(x^i, y^i)]_{\Phi_y}\right\|_2\ge c_H\left\|[\nabla_y f(x^i, y^i)]_{\Phi_y}\right\|_2,\]
where $c_H \in [0, 1]$.  It follows that 
\begin{eqnarray}
\left\|[[\nabla_x f(x^i, y^i)]_{\Gamma_x},[\nabla_y f(x^i, y^i)]_{\Gamma_y}\right\|_2 \ge \nonumber \\ c_H  \left\|[[\nabla_x f(x^i, y^i)]_{\Phi_x},[\nabla_y f(x^i, y^i)]_{\Phi_y}\right\|_2. \label{lemma-comp-1}
\end{eqnarray}
The component $\left\|[[\nabla_x f(x^i, y^i)]_{\Gamma_x},[\nabla_y f(x^i, y^i)]_{\Gamma_y}\right\|_2$ can then be lower bounded as
\begin{eqnarray}
\left\|[[\nabla_x f(x^i, y^i)]_{\Gamma_x},[\nabla_y f(x^i, y^i)]_{\Gamma_y}\right\|_2
\ge \nonumber \\ c_H \left\|[[\nabla_x f(x^i, y^i)]_{\Phi_x},[\nabla_y f(x^i, y^i)]_{\Phi_y}\right\|_2 = \nonumber \\
 c_H\Big(\left\|[\nabla_x f(x^i, y^i)]_{\Phi_x} - [\nabla_x f(x^*, y^*)]_{\Phi_x} + [\nabla_x f(x^*, y^*)]_{\Phi_x}\right\|_2^2 \nonumber \\
 +  \left\|[\nabla_y f(x^i, y^i)]_{\Phi_y} - [\nabla_y f(x^*, y^*)]_{\Phi_y} + [\nabla_y f(x^*, y^*)]_{\Phi_y}\right\|_2^2\Big)^{1/2} \ge  \nonumber \\
 c_H\Big(\left\|[\nabla_x f(x^i, y^i)]_{\Phi_x} - [\nabla_x f(x^*, y^*)]_{\Phi_x}\right\|_2^2  
 + \nonumber\\ \left\|[\nabla_y f(x^i, y^i)]_{\Phi_y} - [\nabla_y f(x^*, y^*)]_{\Phi_y}\right\|_2^2 \Big)^{1/2}  - \nonumber\\
 c_H\left\|[\nabla_x f(x^*, y^*)]_{\Phi_x},[\nabla_y f(x^*, y^*)]_{\Phi_y}\right\|_2 \ge \nonumber \\
  c_H\Big(\left\|[\nabla_x f(x^i, y^i)]_{\Phi_x} - [\nabla_x f(x^*, y^*)]_{\Phi_x}\right\|_2^2  
 + \nonumber \\ \left\|[\nabla_y f(x^i, y^i)]_{\Phi_y} - [\nabla_y f(x^*, y^*)]_{\Phi_y}\right\|_2^2 \Big)^{1/2}  - \nonumber\\
c_H\left\| [\nabla_x f(x^*, y^*)]_{\Phi_x}\right\|_2 - c_H\left\|[\nabla_y f(x^*, y^*)]_{\Phi_y}\right\|_2 \ge\nonumber \\
 c_H \frac{1-\rho}{\xi} \Big(\left\|x^* - x^i\right\|_2^2 + \left\|y^* - y^i\right\|_2^2 \Big)^{1/2} - \nonumber\\ c_H\left\| [\nabla_x f(x^*, y^*)]_{\Phi_x}\right\|_2 - c_H\left\|[\nabla_y f(x^*, y^*)]_{\Phi_y}\right\|_2 \ge \nonumber \\
 c_H \frac{1-\rho}{\xi} \Big(\left\|x^* - x^i\right\|_2^2 + \left\|y^* - y^i\right\|_2^2 \Big)^{1/2} - c_H\cdot \varepsilon_x - c_H\cdot\varepsilon_y = \nonumber\\
 c_H \frac{1-\rho}{\xi} \left\|[r^i_x, r^i_y]\right\|_2 - c_H\cdot\varepsilon_x - c_H\cdot\varepsilon_y,
\nonumber 
\end{eqnarray}
where the first inequality follows from the inequality~(\ref{lemma-comp-1}); the second and the third inequalities follow from the use of the triangle inequalities: $\|a\|_2 + \|b\|_2 \le \|a + b\|_2\ge \|a\|_2 - \|b\|_2$ for any vectors $a$ and $b$); the fourth inequality follows from the inequalities~(\ref{RSC-RSS-ext71}) in Lemma~\ref{rss-rsc-properties}, given that $\Phi_x\in \mathbb{M}(k)$ and $|\Phi_y| \le s$; and the fifth inequality follows from the definitions of $\varepsilon_x$ and $\varepsilon_y$ in Theorem~\ref{theorem:accuracy}. 

The component $\left\|[[\nabla_x f(x^i, y^i)]_{\Gamma_x},[\nabla_y f(x^i, y^i)]_{\Gamma_y}\right\|_2$ can be upper bounded as
\begin{eqnarray}
\left\|[\nabla_x f(x^i, y^i)]_{\Gamma_x},[\nabla_y f(x^i, y^i)]_{\Gamma_y}\right\|_2 = \nonumber \\
 \frac{1}{\xi}\Big(\left\|[\xi \nabla_x f(x^i, y^i)]_{\Gamma_x} - \xi [\nabla_x f(x^*, y^*)]_{\Gamma_x} + \xi [\nabla_x f(x^*, y^*)]_{\Gamma_x}\right\|_2^2  + \nonumber \\
\left\|\xi[\nabla_y f(x^i, y^i)]_{\Gamma_y} - \xi[\nabla_y f(x^*, y^*)]_{\Gamma_y} +\xi [\nabla_y f(x^*, y^*)]_{\Gamma_y}\right\|_2^2\Big)^{1/2} \le \nonumber \\
 \frac{1}{\xi}\Big(\left\|\xi[\nabla_x f(x^i, y^i)]_{\Gamma_x} - \xi[\nabla_x f(x^*, y^*)]_{\Gamma_x} - [{r_x^i}]_{\Gamma_x} + [{r_x^i}]_{\Gamma_x}\right\|_2^2  +  \nonumber \\
\left\|\xi[\nabla_y f(x^i, y^i)]_{\Gamma_y} - \xi[\nabla_y f(x^*, y^*)]_{\Gamma_y} - [{r_y^i}]_{\Gamma_y} + [{r_y^i}]_{\Gamma_y}\right\|_2^2\Big)^{1/2} + \nonumber \\
\left\|[\nabla_x f(x^*, y^*)]_{\Gamma_x}\right\|_2 + \left\|[\nabla_y f(x^*, y^*)]_{\Gamma_y}\right\|_2 \le \nonumber\\
 \frac{1}{\xi}\Big(\left\|\xi[\nabla_x f(x^i, y^i)]_{\Gamma_x} - \xi[\nabla_x f(x^*, y^*)]_{\Gamma_x} - [{r_x^i}]_{\Gamma_x}\right\|_2^2    + \nonumber \\
\left\|\xi[\nabla_y f(x^i, y^i)]_{\Gamma_y} - \xi[\nabla_y f(x^*, y^*)]_{\Gamma_y} - [{r_y^i}]_{\Gamma_y}\right\|_2^2\Big)^{1/2}+  \nonumber \\ 
\frac{1}{\xi}\left\|[[{r_x^i}]_{\Gamma_x}, [{r_y^i}]_{\Gamma_y}]\right\|_2 + \left\|[\nabla_x f(x^*, y^*)]_{\Gamma_x}\right\|_2 + \left\|[\nabla_y f(x^*, y^*)]_{\Gamma_y}\right\|_2\le  \nonumber 
\end{eqnarray}
\begin{eqnarray}
\frac{1}{\xi}\Big(\left\|\xi[\nabla_x f(x^i, y^i)]_{\Gamma_x\cup \Theta_x} - \xi[\nabla_x f(x^*, y^*)]_{\Gamma_x\cup \Theta_x} - [{r_x^i}]_{\Gamma_x\cup \Theta_x}\right\|_2^2 +  \nonumber \\
\left\|\xi[\nabla_y f(x^i, y^i)]_{\Gamma_y\cup \Theta_y} - \xi[\nabla_y f(x^*, y^*)]_{\Gamma_y\cup \Theta_y} - [{r_y^i}]_{\Gamma_y\cup \Theta_y}\right\|_2^2\Big)^{1/2} + \nonumber \\
\frac{1}{\xi} \left\|[[{r_x^i}]_{\Gamma_x}, [{r_y^i}]_{\Gamma_y}]\right\|_2   
  + \left\|[\nabla_x f(x^*, y^*)]_{\Gamma_x}\right\|_2+ \left\|[\nabla_y f(x^*, y^*)]_{\Gamma_y}\right\|_2 = \nonumber\\
\frac{1}{\xi}\Big(\left\|\xi[\nabla_x f(x^i, y^i)]_{\Gamma_x\cup \Theta_x} - \xi[\nabla_x f(x^*, y^*)]_{\Gamma_x\cup \Theta_x} - {r_x^i}\right\|_2^2  + \nonumber \\
\left\|\xi[\nabla_y f(x^i, y^i)]_{\Gamma_y\cup \Theta_y} - \xi[\nabla_y f(x^*, y^*)]_{\Gamma_y\cup \Theta_y} - {r_y^i}\right\|_2^2\Big)^{1/2} + \nonumber \\
\frac{1}{\xi}\left\|[[{r_x^i}]_{\Gamma_x}, [{r_y^i}]_{\Gamma_y}]\right\|_2 + \left\|[\nabla_x f(x^*, y^*)]_{\Gamma_x}\right\|_2   + \left\|[\nabla_y f(x^*, y^*)]_{\Gamma_y}\right\|_2  \le \nonumber \\
  \frac{\rho}{\xi}  \left\|[r_x^i, r_y^i]\right\|_2   +   \nonumber \\
\frac{1}{\xi}\left\|[[{r_x^i}]_{\Gamma_x},[{r_y^i}]_{\Gamma_y}]\right\|_2 + \left\|[\nabla_x f(x^*, y^*)]_{\Gamma_x}\right\|_2  + \left\|[\nabla_y f(x^*, y^*)]_{\Gamma_y}\right\|_2 \le  \nonumber\\
 \frac{\rho}{\xi}  \left\|[r_x^i, r_y^i]\right\|_2 
  +   \frac{1}{\xi}\left\|[[{r_x^i}]_{\Gamma_x},[{r_y^i}]_{\Gamma_y}]\right\|_2 + \varepsilon_x + \varepsilon_y,
\nonumber 
\end{eqnarray}
where the first and second inequalities follow from the use of the triangle inequality; the second equality follows from the fact that $\text{supp}(r_x^i)\subseteq \Gamma_x \cup \Theta_x$ and $\text{supp}(r_y^i)\subseteq \Gamma_y \cup \Theta_y$; the fourth inequality follows from the bound~(\ref{RSC-RSS-ext3}) in Lemma~\ref{rss-rsc-properties}, given that $\text{supp}(r_x^i)\subseteq \Gamma_x \cup \Theta_x \subseteq \mathbb{M}(3k)$ and $|\text{supp}(r_y^i)|\le |\Gamma_y \cup \Theta_y| \le 4s$; and the last inequality follows from the definitions of $\varepsilon_x$ and $\varepsilon_y$, given that $\Gamma_x \in \mathbb{M}(k)$ and $\|\Gamma_y\|_0 \le 2s$.

Combining the above two upper bounds and grouping terms, we have 
\begin{eqnarray}
\left\|[[{r_x^i}]_{\Gamma_x}, [{r_y^i}]_{\Gamma_y}]\right\|_2 \ge \alpha_0 \left\|[r_x^i, r_y^i]\right\|_2 - \beta_0 (\varepsilon_x + \varepsilon_y),\label{lemma-comp-2}
\end{eqnarray}
where $\alpha_0 =c_H (1-\rho) -  \rho = c_H - \rho \Big(c_H+1\Big)$, $\rho = \sqrt{\Big(1 + (\xi\gamma^+)^2 - 2 \xi\gamma^-\Big)}$, and $\beta_0 = (c_H+1) \xi$. Let $\xi = \gamma^-/(\gamma^+)^2$, then $\rho  = \sqrt{1 - \Big(\gamma^-/\gamma^+\Big)^2}$. We assume that $\delta$ is small enough such that $c_H \ge \frac{\rho}{1 - \rho}$ and  $\alpha_0 > 0$. We consider two cases: 
\vspace{2mm}

\textbf{Case 1}: The value of $\left\|[r_x^i, r_y^i]\right\|_2$ satisfies the condition: 
\[\alpha_0 \left\|[r_x^i, r_y^i]\right\|_2 \le \beta_0 (\varepsilon_x + \varepsilon_y).\]
Then  we have
\[\left\|[[{r_x^i}]_{\Gamma_x}, [{r_y^i}]_{\Gamma_y}]\right\|_2 \le \frac{\beta_0 (\varepsilon_x + \varepsilon_y)}{\alpha_0}\]

\textbf{Case 2}:  The value of $\left\|[r_x^i, r_y^i]\right\|_2$ satisfies the condition: 
\[\alpha_0 \left\|[r_x^i, r_y^i]\right\|_2 \ge \beta_0 (\varepsilon_x + \varepsilon_y).\]
 Rewriting the inequality, we get 
\[ \left\|[[{r_x^i}]_{\Gamma_x}, [{r_y^i}]_{\Gamma_y}]\right\|_2 \ge \left\|[r_x^i, r_y^i]\right\|_2\left(\alpha_0 - \frac{\beta_0 (\varepsilon_x + \varepsilon_y)}{\left\|[r_x^i, r_y^i]\right\|_2}\right)
\]
and 
\[
\left\|[[{r_x^i}]_{\Gamma_x}, [{r_y^i}]_{\Gamma_y}]\right\|_2^2 \ge \|[r_x^i, r_y^i]\|_2^2\left(\alpha_0 - \frac{\beta_0 (\varepsilon_x + \varepsilon_y)}{\|[r_x^i, r_y^i]\|_2}\right)^2. 
\]
Moreover, we also have 
\[
\left\|[{r_x^i}, {r_y^i}]\right\|_2^2 = \left\|[[{r_x^i}]_{\Gamma_x}, [{r_y^i}]_{\Gamma_y}]\right\|_2^2 + \left\|[[{r_x^i}]_{\Gamma_x^c},  [{r_y^i}]_{\Gamma_y^c}]\right\|_2^2, 
\]
and 
\begin{eqnarray}
\left\|[[{r_x^i}]_{\Gamma_x^c},  [{r_y^i}]_{\Gamma_y^c}]\right\|_2^2 = \left\|[{r_x^i}, {r_y^i}]\right\|_2^2 - \left\|[[{r_x^i}]_{\Gamma_x}, [{r_y^i}]_{\Gamma_y}]\right\|_2^2. \nonumber 
\end{eqnarray}
Therefore, we obtain 
\begin{eqnarray}
\left\|[[{r_x^i}]_{\Gamma_x^c},  [{r_y^i}]_{\Gamma_y^c}]\right\|_2^2 \le  \left\|[{r_x^i}, {r_y^i}]\right\|_2^2 \left( 1 - \left(\alpha_0 - \frac{\beta_0 (\varepsilon_x + \varepsilon_y)}{\left\|[{r_x^i}, {r_y^i}]\right\|_2^2}\right)^2 \right).  \nonumber 
\end{eqnarray}

We can simplify the right hand side using the following geometric argument, adapted from~\cite{lee2010admira}. Denote $\omega_0 = \alpha_0 - \frac{\beta_0 (\varepsilon_x + \varepsilon_y)}{\|[{r_x^i}, {r_y^i}]\|_2^2}$. Then, $0 < \omega_0 < 1$ because $\alpha_0 \|[r_x^i, r_y^i]\|_2 \ge \beta_0 (\varepsilon_x + \varepsilon_y)$ and $\alpha_0 < 1$. For a free parameter $0 \le \omega \le 1$, a straightforward calculation yields that 
\begin{eqnarray}
\sqrt{1 - \omega_0^2} \le \frac{1}{\sqrt{1 - \omega^2}} - \frac{\omega}{\sqrt{1 - \omega^2}} \omega_0. \nonumber 
\end{eqnarray}
Therefore, substituting into the bound for $\|[[{r_x^i}]_{\Gamma_x^c},  [{r_y^i}]_{\Gamma_y^c}]\|_2$, we get
\begin{eqnarray}
\left\|[[{r_x^i}]_{\Gamma_x^c},  [{r_y^i}]_{\Gamma_y^c}]\right\|_2 \le \nonumber \\
 \left\|[{r_x^i}, {r_y^i}]\right\|_2 \left( \frac{1}{\sqrt{1 - \omega^2}} - \frac{\omega}{\sqrt{1 - \omega^2}} 
 \left(\alpha_0- \frac{\beta_0 (\varepsilon_x + \varepsilon_y)}{\left\|[{r_x^i}, {r_y^i}]\right\|_2^2} \right) \right) \nonumber \\
 = \frac{1 - \omega \alpha_0}{\sqrt{1 - \omega^2}}  \|[{r_x^i}, {r_y^i}]\|_2 + \frac{\omega \beta_0 (\varepsilon_x + \varepsilon_y)}{\sqrt{1 - \omega^2}}
 \nonumber 
\end{eqnarray}

The coefficient preceding $\|[{r_x^i}, {r_y^i}]\|_2$ determines the overall convergence rate, and the minimum value of the coefficient is attained by setting $\omega = \alpha_0$. Substituting, we obtain
\begin{eqnarray}
\|[[{r_x^i}]_{\Gamma_x^c},  [{r_y^i}]_{\Gamma_y^c}]\|_2  
\le \sqrt{1 - \alpha_0^2}  \|[{r_x^i}, {r_y^i}]\|_2 + \frac{\alpha_0  \beta_0 (\varepsilon_x + \varepsilon_y) }{\sqrt{1 - \alpha_0^2}} \label{lemma-comp-3}
\end{eqnarray}

Combining the mutually exclusively cases~(\ref{lemma-comp-2}) and~(\ref{lemma-comp-3}), we obtain 
\begin{eqnarray}
\|[[{r_x^i}]_{\Gamma_x^c},  [{r_y^i}]_{\Gamma_y^c}]\|_2    
\le \sqrt{1 - \alpha_0^2}   \|[{r_x^i}, {r_y^i}]\|_2 + \Big(\frac{\alpha_0 \beta_0  }{1 - \alpha_0^2} + \frac{\beta_0}{\alpha_0}\Big) (\varepsilon_x + \varepsilon_y).  \nonumber
\end{eqnarray}

\end{proof}

\begin{lemma}[Properties of RSC/RSS]
If $f$ satisfies the $(\mathbb{M}(k), s,$ $\gamma^-, \gamma^+)$-$RSS/RSC$, then for every $x, x^\prime \in \mathbb{R}^n$ and $S_x \in \mathbb{M}(4k)$ with $\text{supp}(x)$, $\text{supp}(x^\prime)\in \mathbb{M}(2k)$, and every $y, y^\prime \in \mathbb{R}^p$ and $S_y \subseteq \{1, \cdots, p\}$ with $|\text{supp}(y)|$, $|\text{supp}(y^\prime)|\le 2s$ and $|S_y| \le 4s$, the following inequalities hold:
\begin{itemize}
\item Part 1:
\begin{eqnarray}
\gamma^-(\|x - x^\prime\|_2^2 + \|y - y^\prime\|_2^2) \nonumber \\
\le -\left\langle [\nabla_x f(x, y)]_{S_x} - [\nabla_x f(x^\prime, y^\prime)]_{S_x}, x - x^\prime\right\rangle \nonumber \\ -\left\langle [\nabla_y f(x, y)]_{S_y} - [\nabla_y f(x^\prime, y^\prime)]_{S_y}, y - y^\prime\right\rangle \nonumber \\
\le \gamma^+(\|x - x^\prime\|_2^2 + \|y - y^\prime\|_2^2), \label{RSC-RSS-ext1}
\end{eqnarray}
\item Part 2: 
\begin{eqnarray}
\Big(\|\nabla_x f(x, y) - \nabla_x f(x^\prime, y^\prime)\|_2^2 + \|\nabla_y f(x, y) - \nabla_y f(x^\prime, y^\prime)\|_2^2\Big) \nonumber \\
\le (\gamma^+)^2\Big(\|x - x^\prime\|_2^2 + \|y - y^\prime\|_2^2\Big), \label{RSC-RSS-ext2}
\end{eqnarray}
\item Part 3: For any $\xi \le 2\frac{\gamma^-}{{\gamma^+}^2}$, we have 
\begin{eqnarray}
\|\xi [\nabla_x f(x, y)]_{S_x} - \xi[\nabla_x f(x^\prime, y^\prime)]_{S_x} - (x - x^\prime)\|^2_2 \nonumber\\
+ \|\xi[\nabla_y f(x, y)]_{S_y} - \xi[\nabla_y f(x^\prime, y^\prime)]_{S_y} - (y - y^\prime)\|^2_2\nonumber \\
\le  \rho^2 \Big(\|x - x^\prime\|_2^2 + \|y - y^\prime\|_2^2\Big), \label{RSC-RSS-ext3}
\end{eqnarray}
and 
\begin{eqnarray}
\|\xi[\nabla_x f(x, y)]_{S_x} - \xi[\nabla_x f(x^\prime, y^\prime)]_{S_x} - (x - x^\prime)\|_2 \nonumber\\
+ \|\xi[\nabla_y f(x, y)]_{S_y} - \xi[\nabla_y f(x^\prime, y^\prime)]_{S_y} - (y - y^\prime)\|_2 \nonumber \\
\le \sqrt{2}\rho \Big(\|x - x^\prime\|_2 + \|y - y^\prime\|_2\Big), \label{RSC-RSS-ext4}
\end{eqnarray}
where $\rho = \sqrt{1 + (\xi\gamma^+)^2 - 2 \xi\gamma^-}$. The condition $\xi \le 2\frac{\gamma^-}{{\gamma^+}^2}$ ensures that $\rho \le 1$. In particular, if $\xi = \frac{\gamma^-}{{\gamma^+}^2}$, then $\rho =\sqrt{1 - (\gamma^-/\gamma^+)^2}$.

\item Part 4: 
\begin{eqnarray}
 \frac{1-\rho}{\xi} \Big(\|x - x^\prime\|_2^2 + \|y - y^\prime\|_2^2\Big)^{1/2} \le\nonumber\\ 
  \Big(\|[\nabla_x f(x, y)]_{S_x} - [\nabla_x f(x^\prime, y^\prime)]_{S_x}\|_2^2 + \nonumber \\
 \|[\nabla_y f(x, y)]_{S_y} - [\nabla_y f(x^\prime, y^\prime)]_{S_y}\|_2^2\Big)^{1/2}\nonumber \\  \le \frac{1+\rho}{\xi} \Big(\|x - x^\prime\|_2^2 + \|y - y^\prime\|_2^2\Big)^{1/2}.\label{RSC-RSS-ext71}
\end{eqnarray}
\end{itemize}\label{rss-rsc-properties}
\end{lemma}
\begin{proof}
The proofs of the inequalities in the four parts are stated as follows:
\begin{itemize}
\item \textbf{Part 1:} Recall that $\text{supp}(x-x^\prime)\subseteq S_x$ and $\text{supp}(y-y^\prime)\subseteq S_y$. By adding two copies of the inequalities~(\ref{RSS-RSC}) with $(x, y)$ and $(x^\prime, y^\prime)$ as described in Definition~\ref{def:RSS-RSC}, we have 
\begin{eqnarray}
\gamma^-(\|x - x^\prime\|_2^2 + \|y - y^\prime\|_2^2) \le \nonumber\\
-\langle [\nabla_x f(x, y)]_{S_x} - [\nabla_x f(x^\prime, y^\prime)]_{S_x}, x - x^\prime\rangle -\nonumber \\
 \langle [\nabla_y f(x, y)]_{S_y} - [\nabla_y f(x^\prime, y^\prime)]_{S_y}, y - y^\prime\rangle = \nonumber\\
 -\langle [\nabla_x f(x, y)] - [\nabla_x f(x^\prime, y^\prime)], x - x^\prime\rangle -\nonumber \\
 \langle [\nabla_y f(x, y)] - [\nabla_y f(x^\prime, y^\prime)], y - y^\prime\rangle
 \le \nonumber \\  
  \gamma^+(\|x - x^\prime\|_2^2 + \|y - y^\prime\|_2^2), \nonumber 
\end{eqnarray}
where 
\begin{eqnarray}
\langle \nabla_x f(x, y) - \nabla_x f(x^\prime, y^\prime), x - x^\prime\rangle = \nonumber \\ \langle \nabla_x f(x, y) - \nabla_x f(x^\prime, y^\prime), [x - x^\prime]_{S_x}\rangle = \nonumber \\
 \langle [\nabla_x f(x, y)]_{S_x} - [\nabla_x f(x^\prime, y^\prime)]_{S_x}, x - x^\prime\rangle \nonumber 
\end{eqnarray}
and 
\begin{eqnarray}
\langle \nabla_y f(x, y) - \nabla_y f(x^\prime, y^\prime), y - y^\prime\rangle = \nonumber \\
\langle \nabla_y f(x, y) - \nabla_y f(x^\prime, y^\prime), [y - y^\prime]_{S_y}\rangle = \nonumber \\
 \langle [\nabla_y f(x, y)]_{S_y} - [\nabla_y f(x^\prime, y^\prime)]_{S_y}, y - y^\prime\rangle. \nonumber 
\end{eqnarray}
\item \textbf{Part 2:} By Theorem 2.1.5 in \cite{nesterov2013introductory}, we have 
\begin{eqnarray}
-\langle \nabla_x f(x, y) - \nabla_x f(x^\prime, y^\prime), x - x^\prime\rangle -\nonumber \\ 
\langle \nabla_y f(x, y) - \nabla_y f(x^\prime, y^\prime), y - y^\prime\rangle \nonumber\\ 
\ge \frac{1}{\gamma^+} \Big(\|\nabla_x f(x, y) - \nabla_x f(x^\prime,y)\|_2^2 + \nonumber\\ 
\|\nabla_y f(x, y) - \nabla_y f(x^\prime, y^\prime)\|_2^2\Big). \nonumber 
\end{eqnarray}
We then have
\begin{eqnarray}
\Big(\|\nabla_x f(x, y) - \nabla_x f(x^\prime, y^\prime)\|_2^2 + \|\nabla_y f(x, y) - \nabla_y f(x^\prime, y^\prime)\|_2^2\Big)^{1/2}\cdot \nonumber \\
(\|x - x^\prime\|_2^2 + \|y - y^\prime\|_2^2)^{1/2}  \nonumber\\
\ge \Big(\|\nabla_x f(x, y) - \nabla_x f(x^\prime, y^\prime)\|_2^2 \cdot \|x - x^\prime\|_2^2 + \nonumber \\
\|\nabla_y f(x, y) - \nabla_y f(x^\prime, y^\prime)\|_2^2\cdot \|y - y^\prime\|_2^2\Big)^{1/2} \nonumber \\ 
\ge -\Big\langle \nabla_x f(x, y) - \nabla_x f(x^\prime, y^\prime), x - x^\prime\Big\rangle -
\Big \langle \nabla_y f(x, y) - \nabla_y f(x^\prime, y^\prime), \nonumber \\  y - y^\prime \Big \rangle\nonumber \\
 \ge \frac{1}{\gamma^+} \Big(\|\nabla_x f(x, y) - \nabla_x f(x^\prime, y^\prime)\|_2^2 +
 \|\nabla_y f(x, y) - \nabla_y f(x^\prime, y^\prime)\|_2^2\Big). \nonumber 
\end{eqnarray}
The above inequalities indicate that 
\begin{eqnarray}
\left(\|x - x^\prime\|_2^2 + \|y - y^\prime\|_2^2\right)^{1/2} 
\ge \frac{1}{\gamma^+} \Big(\|\nabla_x f(x, y) - \nabla_x f(x^\prime, y^\prime)\|_2^2 + \nonumber \\
\|\nabla_y f(x, y) - \nabla_y f(x^\prime, y^\prime)\|_2^2\Big)^{1/2}. \nonumber \end{eqnarray}
We then obtain
\begin{eqnarray}
  \Big(\|\nabla_x f(x, y) - \nabla_x f(x^\prime, y^\prime)\|_2^2 + \|\nabla_y f(x, y) - \nabla_y f(x^\prime, y^\prime)\|_2^2\Big)  \le \nonumber \\
  (\gamma^+)^2\Big(\|x - x^\prime\|_2^2 + \|y - y^\prime\|_2^2\Big). \nonumber 
\end{eqnarray}

\item \textbf{Part 3:} Combining the two bounds~(\ref{RSC-RSS-ext1}) and~(\ref{RSC-RSS-ext2}) and grouping terms, we get
\begin{eqnarray}
\|(\xi [\nabla_x f(x, y)]_{S_x} - \xi[\nabla_x f(x^\prime, y^\prime)]_{S_x}) - (x - x^\prime)\|^2_2 +\nonumber\\
 \|(\xi[\nabla_y f(x, y)]_{S_y} - \xi[\nabla_y f(x^\prime, y^\prime)]_{S_y}) - (y - y^\prime)\|^2_2 =\nonumber \\
 \|\xi[\nabla_x f(x, y)]_{S_x} - \xi[\nabla_x f(x^\prime, y^\prime)]_{S_x}\|_2^2 + \|x - x^\prime\|_2^2 - \nonumber \\
2\langle \xi[\nabla_x f(x, y)]_{S_x} - \xi[\nabla_x f(x^\prime, y^\prime)]_{S_x}, x - x^\prime\rangle + \nonumber \\
\|\xi[\nabla_y f(x, y)]_{S_y} - \xi[\nabla_y f(x^\prime, y^\prime)]_{S_y}\|_2^2 + \|y - y^\prime\|_2^2 \nonumber \\
- 2\langle \xi[\nabla_y f(x, y)]_{S_y} - \xi[\nabla_y f(x^\prime, y^\prime)]_{S_y}, y - y^\prime\rangle\le \nonumber \\
 \Big(1 + \xi^2{\gamma^+}^2\Big)\Big(\|x - x^\prime\|_2^2 + \|y - y^\prime\|_2^2\Big) - \nonumber \\
2\langle \xi[\nabla_x f(x, y)]_{S_x} - \xi[\nabla_x f(x^\prime, y^\prime)]_{S_x}, x - x^\prime\rangle - \nonumber \\
2\langle \xi[\nabla_y f(x, y)]_{S_y} - \xi[\nabla_y f(x^\prime, y^\prime)]_{S_y}, y - y^\prime\rangle\le  \nonumber \\
\Big(1 + \xi^2{\gamma^+}^2\Big)\Big(\|x - x^\prime\|_2^2 + \|y - y^\prime\|_2^2\Big) - \nonumber\\
2 \xi\gamma^-\Big(\|x - x^\prime\|_2^2 + \|y - y^\prime\|_2^2\Big) =  \nonumber \\
\Big(1 + \xi^2{\gamma^+}^2 - 2 \xi\gamma^-\Big)\Big(\|x - x^\prime\|_2^2 + \|y - y^\prime\|_2^2\Big),\label{part3-1}
\end{eqnarray}
where the first inequality follows from the bound~(\ref{RSC-RSS-ext2}), and the  last inequality follows from the bound~(\ref{RSC-RSS-ext1}). 
By combining the inequality~(\ref{part3-1}) and the following inequality
\begin{eqnarray}
\Big(\|\xi[\nabla_x f(x, y)]_{S_x} - \xi[\nabla_x f(x^\prime, y^\prime)]_{S_x} - (x - x^\prime)\|_2 \nonumber\\
+ \|\xi[\nabla_y f(x, y)]_{S_y} - \xi[\nabla_y f(x^\prime, y^\prime)]_{S_y} - (y - y^\prime)\|_2\Big)^2 \le \nonumber \\
2\|(\xi [\nabla_x f(x, y)]_{S_x} - \xi[\nabla_x f(x^\prime, y^\prime)]_{S_x}) - (x - x^\prime)\|^2_2 +\nonumber\\
 2 \|(\xi[\nabla_y f(x, y)]_{S_y} - \xi[\nabla_y f(x^\prime, y^\prime)]_{S_y}) - (y - y^\prime)\|^2_2, \label{part3-2}
\end{eqnarray}
we have  
\begin{eqnarray}
\|\xi[\nabla_x f(x, y)]_{S_x} - \xi[\nabla_x f(x^\prime, y^\prime)]_{S_x} - (x - x^\prime)\|_2 \nonumber\\
+ \|\xi[\nabla_y f(x, y)]_{S_y} - \xi[\nabla_y f(x^\prime, y^\prime)]_{S_y} - (y - y^\prime)\|_2 \nonumber \\
\le \sqrt{2\Big(1  - 2 \xi\gamma^- + (\xi\gamma^+)^2\Big)}\Big(\|x - x^\prime\|_2 + \|y - y^\prime\|_2\Big).\nonumber  \end{eqnarray}

\item \textbf{Part 4:}
Let $\xi = \frac{\gamma^-}{(\gamma^+)^2}$ and $\rho = \Big(1 + (\xi\gamma^+)^2 - 2 \xi\gamma^-\Big)^{1/2}$. We have 
\begin{eqnarray}
\xi^2 \|( [\nabla_x f(x, y)]_{S_x} - [\nabla_x f(x^\prime, y^\prime)]_{S_x})\|^2_2 - \|(x - x^\prime)\|^2_2 \nonumber\\
+ \xi^2\|([\nabla_y f(x, y)]_{S_y} - [\nabla_y f(x^\prime, y^\prime)]_{S_y}) \|^2_2- \|(y - y^\prime)\|^2_2 \le \nonumber \\
\|(\xi [\nabla_x f(x, y)]_{S_x} - \xi[\nabla_x f(x^\prime, y^\prime)]_{S_x}) - (x - x^\prime)\|^2_2 \nonumber\\
+ \|(\xi[\nabla_y f(x, y)]_{S_y} - \xi[\nabla_y f(x^\prime, y^\prime)]_{S_y}) - (y - y^\prime)\|^2_2 \label{part3-2}
\end{eqnarray}
Combining the above inequalities~(\ref{part3-1}) and (\ref{part3-2}), we have 
\begin{eqnarray}
  \Big(\|[\nabla_x f(x, y)]_{S_x} - [\nabla_x f(x^\prime, y^\prime)]_{S_x}\|_2^2 + \nonumber \\
 \|[\nabla_y f(x, y)]_{S_y} - [\nabla_y f(x^\prime, y^\prime)]_{S_y}\|_2^2\Big)^{1/2}\nonumber \\  \le \frac{1+\rho}{\xi} \Big(\|x - x^\prime\|_2^2 + \|y - y^\prime\|_2^2\Big)^{1/2}.
\end{eqnarray}
By combining the inequality~(\ref{part3-1}) and the following inequality: 
\begin{eqnarray}
-\xi^2 \|( [\nabla_x f(x, y)]_{S_x} - [\nabla_x f(x^\prime, y^\prime)]_{S_x})\|^2_2 + \|(x - x^\prime)\|^2_2 \nonumber\\
- \xi^2\|([\nabla_y f(x, y)]_{S_y} + [\nabla_y f(x^\prime, y^\prime)]_{S_y}) \|^2_2 + \|(y - y^\prime)\|^2_2 \le \nonumber \\
\|(\xi [\nabla_x f(x, y)]_{S_x} - \xi[\nabla_x f(x^\prime, y^\prime)]_{S_x}) - (x - x^\prime)\|^2_2 \nonumber\\
+ \|(\xi[\nabla_y f(x, y)]_{S_y} - \xi[\nabla_y f(x^\prime, y^\prime)]_{S_y}) - (y - y^\prime)\|^2_2, \nonumber
\end{eqnarray}
we conclude that 
\begin{eqnarray}
 \frac{1-\rho}{\xi} \Big(\|x - x^\prime\|_2^2 + \|y - y^\prime\|_2^2\Big)^{1/2} \le\nonumber\\ 
  \Big(\|[\nabla_x f(x, y)]_{S_x} - [\nabla_x f(x^\prime, y^\prime)]_{S_x}\|_2^2 + \nonumber \\
 \|[\nabla_y f(x, y)]_{S_y} - [\nabla_y f(x^\prime, y^\prime)]_{S_y}\|_2^2\Big)^{1/2}\nonumber
\end{eqnarray}
\end{itemize}
\end{proof}

\section{Proof of Theorem 3.1}
\label{proof-theorem-3.1}
\subsection{Negative squared error function}

Recall that the negative squared error function has the form:
\[f(x, y) = - \|c - W^\top x - y\|_2^2 - \frac{1}{2} \|x\|_2^2 - \frac{1}{2} \|y\|_2^2,\]
where $x\in \mathbb{R}^n$ and $y \in \mathbb{R}^p$. The following Lemma discusses the RSC/RSS property of the negative squared error function
\begin{lemma}[Negative squared error function]
Let $I_n$ and $I_p$ be the identity matrices of sizes $n\times n$ and $p\times p$, respectively. If the attribute matrix $W \in \mathbb{R}^{n \times p}$ satisfies the condition: $WW^\intercal \preceq b^0 I_n \preceq I_n$ and $W^\intercal W \preceq b^1 I_p \preceq I_p$, for every $x \in [0, 1]^n$ and $y \in [0, 1]^p$, such that $\text{supp}(x) \in \mathbb{M}({k})$ and $\|y\|_0 \le s$, then 
the negative squared error function satisfies the $(\mathbb{M}(k), s, \gamma^-, \gamma^+)$-$RSS/RSC$, where $\gamma^- = 1$  and $\gamma^+ = \max\Big(2b + 2\sqrt{b} + 1, 3 + 2\sqrt{b} \Big)$. 
\end{lemma}
\begin{proof}
Let $x^\prime =  x + \Delta_x$ and $y^\prime = y + \Delta_y$, such that $\text{supp}(x)$, $\text{supp}(x^\prime) \in \mathbb{M}(k)$ and $\|y\|_0, \|y^\prime\|_0 \le s$.  Denote $g(x^\prime, y^\prime, x, y) =  f(x, y) - f(x^\prime, y^\prime)  -   \nabla_x f(x, y)^\intercal(x - x^\prime) - \nabla_y f(x, y)^\intercal(y - y^\prime)$. The component $g(x^\prime, y^\prime, x, y)$ can be upper bounded as 
\begin{eqnarray}
g(x^\prime, y^\prime, x, y) 
= \|W\Delta_x + \Delta_y\|_2^2 + \frac{1}{2}\Delta_x^\intercal \Delta_x + \frac{1}{2}\Delta_y^\intercal \Delta_y \nonumber\\
\le \Big(\|W\Delta_x\|_2 + \|\Delta_y\|_2\Big)^2 + \frac{1}{2}\Delta_x^\intercal \Delta_x + \frac{1}{2}\Delta_y^\intercal \Delta_y \nonumber \\
= \|W\Delta_x\|_2^2 + \|\Delta_y\|_2^2 + 2 \|W\Delta_x\|_2 \|\Delta_y\|_2 + \frac{1}{2}\Delta_x^\intercal \Delta_x + \frac{1}{2}\Delta_y^\intercal \Delta_y\nonumber\\
\le b\|\Delta_x\|_2^2 + \|\Delta_y\|_2^2 + 2 \sqrt{b}\|\Delta_x\|_2 \|\Delta_y\|_2 + \frac{1}{2}\Delta_x^\intercal \Delta_x + \frac{1}{2}\Delta_y^\intercal \Delta_y \nonumber\\
\le b\|\Delta_x\|_2^2 + \|\Delta_y\|_2^2 + \sqrt{b}\|\Delta_x\|_2^2 + \sqrt{b}\|\Delta_y\|_2^2 + \frac{1}{2}\Delta_x^\intercal \Delta_x + \frac{1}{2}\Delta_y^\intercal \Delta_y \nonumber\\
\le (b + \sqrt{b})\|\Delta_x\|_2^2 + (1 + \sqrt{b})\|\Delta_y\|_2^2 + \frac{1}{2}\Delta_x^\intercal \Delta_x + \frac{1}{2}\Delta_y^\intercal \Delta_y \nonumber \\
\le \max\Big(b + \sqrt{b} + 0.5, 1.5 + \sqrt{b} \Big) \Big(\|\Delta_x\|_2^2 + \|\Delta_y\|_2^2\Big) \nonumber\\
= \frac{\gamma^+}{2} \Big(\|\Delta_x\|_2^2 + \|\Delta_y\|_2^2\Big) \nonumber 
\end{eqnarray}
where $\gamma^+ = \max\Big(b + \sqrt{b} + 0.5, 1.5 + \sqrt{b} \Big).$ The component $g(x',y',x,y)$ can also be lower bounded as
\begin{eqnarray}
g(x',y',x,y) = 
  \|W\Delta_x + \Delta_y\|_2^2 + \frac{1}{2}\Delta_x^\intercal \Delta_x + \frac{1}{2}\Delta_y^\intercal \Delta_y \ge \nonumber\\
    0.5\Big(\|\Delta_x\|_2^2 + \|\Delta_y\|_2^2\Big) = \frac{\gamma^-}{2} \Big(\|\Delta_x\|_2^2 + \|\Delta_y\|_2^2\Big)\nonumber,
\end{eqnarray}
where $\gamma^- = 1$. 
\end{proof}

\subsection{Fisher's test statistic function}
Recall that $w_i$ refers to the vector of observations of the $p$ attributes at node $i$, the attribute matrix $W$ is defined as $W = [w_1, \cdots, w_n]^\intercal$, and the Fisher's test statistic is defined as 
\[f(x, y) = x^\intercal Wy - \frac{1}{2} \|x\|_2^2 - \frac{1}{2} \|y\|_2^2,\]
where we consider the soft values of $x$ and $y$: $x \in [0, 1]^n$ and $y \in [0, 1]^p$. We consider the relaxed input domains $[0, 1]^n$ and $ [0, 1]^p$ for $x$ and $y$, instead of their original domains $\{0, 1\}^n$ and $\{0, 1\}^p$, respectively, such that our proposed algorithm \texttt{SG-Pursuit} can be applied to optimize this score function. The following Lemma discusses the RSC/RSS property of the Fisher's test statistic function: 

\begin{lemma}[Fisher's test statistic]
Let $I_n$ and $I_p$ be the identity matrices of sizes $n\times n$ and $p\times p$, respectively. If the attribute matrix $W \in \mathbb{R}^{n \times p}$ satisfies the condition: $WW^\intercal \preceq b^0 I_n \preceq I_n$ and $W^\intercal W \preceq b^1 I_p \preceq I_p$, for every $x \in [0, 1]^n$ and $y \in [0, 1]^p$, such that $\text{supp}(x) \in \mathbb{M}({k})$ and $\|y\|_0 \le s$, then 
the Fisher's test statistic function satisfies the $(\mathbb{M}(k), s, \gamma^-, \gamma^+)$-$RSS/RSC$, where $\gamma^- = \min \Big\{1 - b^0, 1 - b^1\Big\}$  and $\gamma^+ = 2$. \label{fisher-stat}
\end{lemma}
\begin{proof}
Let $x^\prime =  x + \Delta_x$ and $y^\prime = y + \Delta_y$, such that $\text{supp}(x)$, $\text{supp}(x^\prime) \in \mathbb{M}(k)$ and $\|y\|_0, \|y^\prime\|_0 \le s$.  Denote $g(x^\prime, y^\prime, x, y) =  f(x, y) - f(x^\prime, y^\prime)  -   \nabla_x f(x, y)^\intercal(x - x^\prime) - \nabla_y f(x, y)^\intercal(y - y^\prime)$. The component $g(x^\prime, y^\prime, x, y)$ can be upper bounded as 
\begin{eqnarray}
g(x^\prime, y^\prime, x, y) = f(x, y) - f(x + \Delta_x, y + \Delta_y) + \nonumber \\ \nabla_x f(x, y)^\intercal \Delta_x + \nabla_y f(x, y)^\intercal \Delta_y = \nonumber \\
 -\Delta_x^\intercal W \Delta_y + \frac{1}{2}\Delta_x^\intercal \Delta_x + \frac{1}{2}\Delta_y^\intercal \Delta_y = \nonumber \\
\frac{1}{2}\|\Delta_x - W\Delta_y\|_2^2 - \frac{1}{2}\Delta_y^\intercal W^\intercal W\Delta_y + \frac{1}{2} \Delta_y^\intercal \Delta_y \le \nonumber \\
 \frac{1}{2} (\|\Delta_x\|_2 + \|W \Delta_y\|_2)^2 - \frac{1}{2}\Delta_y^\intercal W^\intercal W\Delta_y + \frac{1}{2} \Delta_y^\intercal \Delta_y = \nonumber \\
 \frac{1}{2} \Delta_x^\intercal\Delta_x + \frac{1}{2} \Delta_y^\intercal \Delta_y + \|\Delta_x\|_2 \|W \Delta_y\|_2  \le \nonumber \\
\frac{1}{2} \Delta_x^\intercal\Delta_x + \frac{1}{2} \Delta_y^\intercal \Delta_y + \frac{1}{2} \Delta_x^\intercal\Delta_x + \frac{1}{2}\Delta_y^\intercal W^\intercal W\Delta_y  \le \nonumber \\
\Delta_x^\intercal\Delta_x + \frac{1}{2} \Delta_y^\intercal(W^\intercal W + I) \Delta_y \le \nonumber \\ 
 \max\{1, \frac{1}{2}(b^1 + 1)\}\Big(\|\Delta_x\|_2^2 + \|\Delta_y\|_2^2\Big)  = \nonumber \\
\|\Delta_x\|_2^2 + \|\Delta_y\|_2^2 = \frac{\gamma^+}{2}\|\Delta_x\|_2^2 + \|\Delta_y\|_2^2, \label{eqn:Fisher1}
\end{eqnarray}
where $b^1 \le 1$, $\frac{1}{2}(b^1 + 1) \le 1$, and $\gamma^+ = 2$. The component $g(x^\prime, y^\prime, x, y)$ can be lower bounded as 
\begin{eqnarray}
g(x^\prime, y^\prime, x, y) = 
- \Delta_x^\intercal W \Delta_y + \frac{1}{2}\Delta_x^\intercal \Delta_x + \frac{1}{2}\Delta_y^\intercal \Delta_y \ge \nonumber \\
-\frac{1}{2}(\|\Delta_x\|_2^2 + \|W \Delta_y\|_2^2)  + \frac{1}{2}\Delta_x^\intercal \Delta_x + \frac{1}{2}\Delta_y^\intercal \Delta_y = \nonumber \\
-\frac{1}{2} \|W \Delta_y\|_2^2  + \frac{1}{2}\Delta_y \Delta_y^\intercal = \nonumber\\
\frac{1}{2} \Delta_y^\intercal (I - W^\intercal W) \Delta_y \ge \frac{1}{2} (1 - b^1) \|\Delta_y\|_2^2. \label{eqn:Fisher2}
\end{eqnarray}

We can also obtain the lower bound of $g(x^\prime, y^\prime, x, y)$ as
\begin{eqnarray}
g(x^\prime, y^\prime, x, y) = \nonumber \\
 - \Delta_x^\intercal W \Delta_y + \frac{1}{2}\Delta_x^\intercal \Delta_x + \frac{1}{2}\Delta_y^\intercal \Delta_y = \nonumber \\
\frac{1}{2}\|W^\intercal\Delta_x - \Delta_y\|_2^2 - \frac{1}{2}\Delta_x^\intercal WW^\intercal\Delta_x + \frac{1}{2} \Delta_x^\intercal \Delta_x \ge \nonumber \\
\frac{1}{2}\Big(\|W^\intercal\Delta_x\|_2 - \|\Delta_y\|_2\Big)^2 - \frac{1}{2}\Delta_x^\intercal WW^\intercal\Delta_x + \frac{1}{2} \Delta_x^\intercal \Delta_x = \nonumber 
\end{eqnarray}
\begin{eqnarray}
- \|W^\intercal\Delta_x\|_2 \|\Delta_y\|_2 + \frac{1}{2} \Delta_y^\intercal \Delta_y + \frac{1}{2} \Delta_x^\intercal \Delta_x \ge \nonumber \\
  - \frac{1}{2}\Big(\|W^\intercal\Delta_x\|_2^2 + \|\Delta_y\|_2^2\Big) + \frac{1}{2} \Delta_y^\intercal \Delta_y + \frac{1}{2} \Delta_x^\intercal \Delta_x =  \nonumber \\
\frac{1}{2} \Delta_x^\intercal (I - W W^\intercal) \Delta_x \ge \frac{1}{2} (1-b^0) \|\Delta_x\|_2^2. \label{eqn:Fisher3}
\end{eqnarray}

By combining the inequalities~(\ref{eqn:Fisher2})  and (\ref{eqn:Fisher3}), we obtain 
\begin{eqnarray}
f(x^\prime, y^\prime) - f(x, y)  - \nabla_x f(x, y)^\intercal(x^\prime - x) - \nabla_y f(x, y)^\intercal(y^\prime - y) \nonumber \\
\ge \frac{1}{2}\min\Big\{1 - b^1, 1 - b^0\Big\} \left(\|\Delta_x\|_2^2 + \|\Delta_y\|_2^2\right). \label{eqn:Fisher31}
\end{eqnarray}

By combining  the inequalities~(\ref{eqn:Fisher1})  and (\ref{eqn:Fisher3}), we get  
\begin{eqnarray}
\frac{\gamma^-}{2} (\|\Delta_x\|_2^2 + \|\Delta_y\|_2^2) \le 
g(x^\prime, y^\prime, x, y) \le \frac{\gamma^+}{2}\left(\|\Delta_x\|_2^2 + \|\Delta_y\|_2^2\right), \nonumber 
\end{eqnarray}
where $\gamma^- = \min\Big\{1 - b^0, 1 - b^1\Big\}$ and $\gamma^+ = 2$. 
\end{proof}
In the above lemma, it is required that $b^0$ and $b^1$ are less than $1$. Given that $x \in [0, 1]^n$ and $y \in [0, 1]^p$, the attribute matrix $W$ can be normalized such that $b^0, b^1 \le 1$.

\subsection{Logistic function}
Recall that the logistic function is defined as 
\begin{eqnarray}f(x, y) = \sum_{i=1}^p \Big(y_i\log g(x^\intercal w_i ) + (1-y_i)\log(1 - g(x^\intercal w_i ))\Big) -  \nonumber\\
\frac{1}{2} \|x\|_2^2 - \frac{1}{2} \|y\|_2^2,\nonumber \end{eqnarray}
where $w_i = [w_i(1), \cdots, w_i(n)]^\top$ is the vector of observations of the $i$-th attribute at the $n$ nodes in $\mathbb{V}$, $w_i(j)$ is the observation of the $i$-th attribute at node $j$, $x\in \mathbb{R}^n$ is the vector of the weights (coefficients) of the $n$ nodes in $\mathbb{V}$, and $y \in [0, 1]^n$ is the vector of soft binary variables that indicate the anomalousness of the $p$ attributes, and the $i$-th attribute is anomalous if $y_i > 0$. 

\begin{lemma}[Logistic function]
Let $I_n$ and $I_p$ be the identity matrices of sizes $n\times n$ and $p\times p$, respectively. If the attribute matrix $W \in \mathbb{R}^{n \times p}$ satisfies the condition: $WW^\intercal \preceq b^0 I_n \preceq I_n$ and $W^\intercal W \preceq b^1 I_p \preceq I_p$, for every $x \in [0, 1]^n$ and $y \in [0, 1]^p$, such that $\text{supp}(x) \in \mathbb{M}({k})$ and $\|y\|_0 \le s$, then 
the logistic function satisfies the $(\mathbb{M}(k), s, \gamma^-, \gamma^+)$-$RSS/RSC$, where $\gamma^-_{k, s} = \min\Big\{1 - b^0, 1 - b^1\Big\}$  and $\gamma^+_{k, s} = \max\Big\{2 b^0 + 1, 2\Big\}$. 
\end{lemma}
\begin{proof}
It suffices to prove the RSC/RSS property of the logistic function if the following inequalities hold: 
\begin{eqnarray}
\gamma^- I_{n+p} \preceq - \nabla^2_{x, y}f(x, y)  \preceq \gamma^+ I_{n+p}, 
\end{eqnarray}
where $\nabla^2_{x, y}f(x, y)$ is the Hessian matrix of $f(x, y)$, and $I_{n+p}$ is an identity matrix of size $n+p$ by $n+p$. 

The first-order derivatives of the score function $f(x,y)$ has the following forms: 
\begin{eqnarray}
\nabla_y f(x, y) =   [\log g(x^\intercal w_1), \cdots, \log g(x^\intercal w_p)] ^\intercal  \nonumber\\
- [\log(1 - g(x^\intercal w_1)), \cdots, \log (1 - g(x^\intercal w_p))]^\intercal - y \nonumber
\end{eqnarray}
and 
\begin{eqnarray}
\nabla_x f(x, y) = [(1 - g(x^\intercal w_1)) w_1, \cdots, (1 - g(x^\intercal w_p)w_p)] y  \nonumber \\ 
+ [g(x^\intercal w_1) w_1, \cdots, g(x^\intercal w_p)w_p)] (1 - y) - x. \nonumber 
\end{eqnarray}

The second-order derivatives of the score function has the following forms: 
\begin{eqnarray}
\nabla_x^2  f(x, y)  &=&  -\Big[g(x^\intercal w_1)(1 - g(x^\intercal w_1)) w_1{w_1}^\intercal, 
\cdots, +g(x^\intercal w_p)\nonumber\\&&(1 - g(x^\intercal w_p)) w_p{w_p}^\intercal\Big]y  \nonumber \\&&
-\Big[g(x^\intercal w_1) (1 - g(x^\intercal w_1)) w_1{w_1}^\intercal, \cdots, g(x^\intercal w_p) \nonumber\\ &&
(1 - g(x^\intercal w_p))w_p{w_p}^\intercal )\Big](1 - y) - I_n,\nonumber\\ 
\nabla_{x,y} f(x, y) &=&   \Big[(1- g(x^\intercal w_1))w_1, \cdots, (1 - g(x^\intercal w_p))w_p \Big] \nonumber\\&&
+ \Big[g(x^\intercal w_1)w_1, \cdots, g(x^\intercal w_p)w_p\Big] 
=  \Big[w_1, \cdots, w_p\Big],\nonumber \\
\nabla_y^2  f(x, y) &=& -I_p.\nonumber 
\end{eqnarray}
where $I_n$ and $I_p$ refer to the identity matrices of sizes $n$ by $n$ and $p$ by $p$, respectively. For every $\Delta_x$ and $\Delta_y $, such that $\text{supp}(\Delta_x) \in \mathbb{M}(k)$ and $\|\Delta_y\|_0 \le s$, we obtain 
\begin{eqnarray}
\Delta_x \nabla_x^2  f(x, y) \Delta_x^\intercal = \sum_{i=1}^p g(x^\intercal w_i)(1 - g(x^\intercal w_i)) \Delta_x^\intercal w_i{w_i}^\intercal \Delta_x + \Delta_x^\intercal \Delta_x, \nonumber 
\end{eqnarray}
\begin{eqnarray}
\Delta_x\nabla_{x,y} f(x, y) \Delta_x = - \Delta_x^\intercal [w_1, \cdots, w_p] \Delta_y, \nonumber 
\end{eqnarray}
and 
\begin{eqnarray}
\Delta_y \nabla_y^2  f(x, y) \Delta_y^\intercal = \Delta_y^\intercal \Delta_y.
\end{eqnarray}
It follows that 
\begin{eqnarray}
- [\Delta_x, \Delta_y]^\intercal \nabla^2_{x, y}f(x, y) [\Delta_x,\Delta_y] = \nonumber\\
-\Delta_x \nabla_x^2  f(x, y) \Delta_x^\intercal - \Delta_x \nabla_y^2  f(x, y) \Delta_y^\intercal - 2\Delta_x\nabla_{x,y} f(x, y) \Delta_y^\intercal = \nonumber \\
\sum_{i=1}^p g(x^\intercal w_i)(1 - g(x^\intercal w_i)) \Delta_x^\intercal w_i{w_i}^\intercal \Delta_x - 2\Delta_x^\intercal W \Delta_y +  \nonumber\\
\Delta_x^\intercal \Delta_x + \Delta_y^\intercal \Delta_y \le \nonumber \\
\sum_{i=1}^p \Delta_x^\intercal w_i{w_i}^\intercal \Delta_x - 2 \Delta_x^\intercal W \Delta_y + \Delta_x^\intercal \Delta_x + \Delta_y^\intercal \Delta_y = \nonumber\\
\Delta_x^\intercal WW^\intercal \Delta_x - 2 \Delta_x^\intercal W \Delta_y + \Delta_x^\intercal \Delta_x + \Delta_y^\intercal \Delta_y \le \nonumber\\
\Delta_x^\intercal WW^\intercal \Delta_x + \Delta_y^\intercal \Delta_y +  \Delta_x^\intercal W W^\intercal \Delta_x + 
\Delta_x^\intercal \Delta_x + \Delta_y^\intercal \Delta_y \le \nonumber\\
(2 b^0 + 1) \|\Delta_x\|_2^2 + 2\|\Delta_y\|_2^2 \le \nonumber\\
\max\Big\{2 b^0 + 1, 2\Big\} \Big(\|\Delta_x\|_2^2 + \|\Delta_y\|_2^2\Big)
\nonumber
\end{eqnarray}
where the first inequality follows from the fact that $0 \le g(x^\intercal w_i) \le 1$, the second and third inequalities follow from the use of the triangle inequality, and the third inequality follows from the assumed property of the attribute matrix $W$: $WW^\intercal \preceq b^0 I_n$.  The component $- [\Delta_x, \Delta_y]^\intercal \nabla^2_{x, y}f(x, y) [\Delta_x,\Delta_y]$ can lower bounded as  
\begin{eqnarray}
- [\Delta_x, \Delta_y]^\intercal \nabla^2_{x, y}f(x, y) [\Delta_x,\Delta_y] = \nonumber\\
 -\Delta_x \nabla_x^2  f(x, y) \Delta_x^\intercal - \Delta_x \nabla_y^2  f(x, y) \Delta_y^\intercal - 2\Delta_x\nabla_{x,y} f(x, y) \Delta_y^\intercal \ge \nonumber \\
- 2\Delta_x^\intercal W \Delta_y + \Delta_x^\intercal \Delta_x + \Delta_y^\intercal \Delta_y \ge  \nonumber \\
\Delta_x^\intercal \Delta_x + \Delta_y^\intercal \Delta_y - \Delta_y^\intercal \Delta_y -  \Delta_x^\intercal W W^\intercal \Delta_x \ge  \nonumber \\
(1 - b^0) \Delta_x^\intercal \Delta_x,
\nonumber 
\end{eqnarray}
and 
\begin{eqnarray}
-\Delta_x \nabla_x^2  f(x, y) \Delta_x^\intercal - \Delta_x \nabla_y^2  f(x, y) \Delta_y^\intercal - 2\Delta_x\nabla_{x,y} f(x, y) \Delta_y^\intercal \ge  \nonumber \\
- 2\Delta_x^\intercal W \Delta_y + \Delta_x^\intercal \Delta_x + \Delta_y^\intercal \Delta_y \ge  \nonumber \\
\Delta_x^\intercal \Delta_x + \Delta_y^\intercal \Delta_y - \Delta_x^\intercal \Delta_x -  \Delta_y^\intercal   W^\intercal W\Delta_y \ge  \nonumber \\
(1 - b^1) \Delta_y^\intercal \Delta_y. \nonumber 
\end{eqnarray}
A refined lower bound can be obtained as
\begin{eqnarray}
- [\Delta_x, \Delta_y]^\intercal \nabla^2_{x, y}f(x, y) [\Delta_x,\Delta_y]  \ge  
\min\Big\{1 - b^0, 1 - b^1\Big\}\cdot \nonumber\\ 
 \Big([\Delta_x, \Delta_y]^\intercal [\Delta_x,\Delta_y]\Big). \nonumber 
\end{eqnarray}
\end{proof}

\subsection{Elevated mean scan statistic function}
Recall that the elevated mean scan statistic function is defined as 
\[f(x, y) = x^\intercal Wy / \sqrt{x^\intercal \textbf{1}} - \frac{1}{2} \|x\|_2^2 - \frac{1}{2} \|y\|_2^2,\]
where $x\in [0, 1]^n$ and $y \in [0, 1]^p$. We consider the relaxed input domains $[0, 1]^n$ and $ [0, 1]^p$ for $x$ and $y$, instead of their original domains $\{0, 1\}^n$ and $\{0, 1\}^p$, respectively, such that our proposed algorithm \texttt{SG-Pursuit} can be applied to optimize this score function. The corresponding optimization problem
\begin{eqnarray}
\max_{x \in [0, 1]^n, y\in [0, 1]^p} f(x, y)\ \ \ s.t.\ \ \ \text{supp}(x) \in \mathbb{M}(k), \|y\|_0 \le s, \nonumber 
\end{eqnarray}
has an equivalent formulation (with added constraint)
\begin{eqnarray}
\max_{x \in [0, 1]^n, y\in [0, 1]^p} f(x, y)\ \ \ s.t.\ \ \ \text{supp}(x) \in \mathbb{M}(k), 1^\intercal x = r, \|y\|_0 \le s, \nonumber
\end{eqnarray}
where $r$ refers to the true sparsity of $x$. In practice, $r$ is unknown, but can be identified by considering the $k$ possible numbers: $\{1, 2, \cdots, k\}$, where $k\ll n$. 
The following Lemma discusses the RSC/RSS property of this score function. 
\begin{lemma}[Elevated mean scan statistic]
Let $I_n$ and $I_p$ be the identity matrices of sizes $n\times n$ and $p\times p$, respectively. If the true sparsity of $x$ is given as $r$ and the attribute matrix $W \in \mathbb{R}^{n \times p}$ satisfies the condition: $\frac{1}{r}\cdot WW^\intercal \preceq b^0 I_n \preceq I_n$ and $\frac{1}{r}\cdot W^\intercal W \preceq b^1 I_p \preceq I_p$, for every $x \in [0, 1]^n$ and $y \in [0, 1]^p$, such that $\text{supp}(x) \in \mathbb{M}({k})$ and $\|y\|_0 \le s$, then 
the elevated mean scan statistic function satisfies the $(\mathbb{M}(k), s, \gamma^-, \gamma^+)$-$RSS/RSC$, where $\gamma^- = \min \Big\{1 - b^0, 1 - b^1\Big\} \Big(\|\Delta_x\|_2^2 + \|\Delta_y\|_2^2\Big)$  and $\gamma^+ = 2$. 
\end{lemma}
\begin{proof}
The statistic function $f(x, y)$ can be reformulated as an Fisher's test statistic function: 
\[f(x, y) = x^\intercal\tilde{W}y  - \frac{1}{2} \|x\|_2^2 - \frac{1}{2} \|y\|_2^2,\]
where $\tilde{W} = \frac{W}{\sqrt{r}}$. This lemma follows from Lemma~\ref{fisher-stat}. 

\end{proof}

\bibliographystyle{abbrv}

\end{document}